\documentclass{article}
\usepackage{arxiv}
\usepackage{etoolbox}
\usepackage[T1]{fontenc}
\usepackage{natbib}
\setcitestyle{authoryear,round,citesep={;},aysep={,},yysep={;}}
\usepackage{amsmath,amssymb,amsfonts,amsthm,mathtools}
\usepackage{booktabs,tabularx,multirow}
\usepackage[table]{xcolor}
\usepackage{graphicx,subcaption,microtype,enumitem,placeins,flafter,needspace,url,xurl,hyperref}
\definecolor{linkblue}{HTML}{3569A8}
\definecolor{palegreen}{HTML}{EAF3F0}
\definecolor{paleblue}{HTML}{EAF1F7}
\definecolor{palegray}{HTML}{F3F5F7}
\hypersetup{colorlinks=true,linkcolor=linkblue,citecolor=linkblue,urlcolor=linkblue}
\newtheorem{theorem}{Theorem}[section]
\newtheorem{proposition}[theorem]{Proposition}
\newtheorem{corollary}[theorem]{Corollary}
\DeclareMathOperator*{\argmax}{arg\,max}
\DeclareMathOperator{\osc}{osc}
\DeclareMathOperator{\LSE}{LSE}
\newcommand{\R}{\mathbb R}

\newcommand{\Ee}{\mathbb E}

\title{Prediction Dynamics in\\Depth-Recurrent Language Models}
\renewcommand{\shorttitle}{Prediction Dynamics in Depth-Recurrent Language Models}
\renewcommand{\undertitle}{}
\renewcommand{\headeright}{}
\date{}
\newcommand{\PreprintPDFAuthors}{Xinyue Luo, Fei Yu}
\author{Xinyue Luo \qquad Fei Yu\\
{\normalfont Ant Group}\\
{\normalfont\small\href{mailto:mangduo.lxy@antgroup.com}{\texttt{mangduo.lxy@antgroup.com}}}}

\hypersetup{pdftitle={Prediction Dynamics in Depth-Recurrent Language Models},pdfauthor={\PreprintPDFAuthors},pdfsubject={Prediction dynamics across recurrent depth},pdfkeywords={depth-recurrent language models, prediction dynamics, answer preservation}}
\makeatletter
\patchcmd{\@maketitle}{\textsc{\undertitle}\\}{}{}
  {\PackageError{arxiv-title}{Could not remove the subtitle line}{Check arxiv.sty.}}
\patchcmd{\@maketitle}
  {\vskip 0.4in \@minus 0.1in \center{\@date} \vskip 0.2in}
  {\par\vskip 0.2in}{}
  {\PackageError{arxiv-title}{Could not remove the date block}{Check arxiv.sty.}}
\makeatother
\begin{document}
\maketitle

\begin{abstract}
Depth-recurrent language models refine predictions through repeated latent updates. Why can intermediate answers agree with the endpoint while their scores continue to change? We derive a sharp margin characterization that decomposes the conservatism of a magnitude bound into common translation, direction relative to the winner, and the pairing of each competitor's update with its score gap. Across Huginn-3.5B and Ouro-1.4B, accounting for update direction and competitor pairing reduces the mean earliest qualifying depth by a further 22.5--34.4\% of the total depth beyond translation removal under full answer-text scoring. This retrospective comparison uses completed trajectories. Substantial contributions also occur under label scoring. For shared predictive distributions, we separate common and contrast motion orthogonally and express the common component through candidate-set mass and within-set concentration. Common and contrast energies can attenuate at different rates, allowing a growing preference-change share to coexist with shrinking absolute updates. These findings explain finite-depth answer preservation through the geometry and composition of observed score changes.
\end{abstract}

\section{Introduction}

Additional inference-time computation can improve language-model reasoning, with its value depending on allocation across problems \citep{snell2025testtime}. Latent approaches place computation before discrete answer generation, through pause positions \citep{goyal2024pause}, continuous intermediate states \citep{hao2025coconut}, or repeated applications of shared transformer blocks \citep{geiping2025latent,zhu2025ouro}. Depth-recurrent models make the number of updates an inference-time variable. A benchmark score summarizes the outcome at one budget; intermediate predictions reveal how that outcome develops. We study how additional latent computation reshapes answer preferences and which components of this evolution affect the final prediction.

A large score update may shift all candidates together, enlarge the winner's lead, or favor a competitor whose gap remains too large to close. Such updates can preserve the same answer despite substantial score changes. We derive an exact decomposition of the gap between a magnitude-based margin bound and the realized endpoint margin into translation, winner-directed motion, and competitor pairing. We then establish the corresponding ordering of earliest qualifying depths. Across Huginn and Ouro, direction and pairing account for substantial conservatism remaining after translation removal under both label selection and complete answer-text scoring. The empirical contribution is the size and composition of this gap along observed trajectories. We also characterize what common score motion contains. For scores from one predictive distribution, we first separate common and contrast components orthogonally, then express the common component through candidate-set mass and within-set concentration, including their interaction. Their depth trajectories show how the share of preference change can grow while its absolute magnitude falls. Together, the two analyses connect the composition of score changes to their relation with the decision boundary.

Paired scoring and precision comparisons examine how these findings depend on the readout construction and numerical arithmetic. Coordinate-permutation references characterize update alignment and joint variation, while held-out tests assess prediction. These comparisons place the descriptive findings alongside their empirical limits.

\section{Related Work}

Latent computation makes the amount of inference separable from parameter count. Adaptive computation time and Universal Transformers introduce iterative computation \citep{graves2016act,dehghani2019universal}, while looped constructions establish its expressive power \citep{giannou2023programmable,saunshi2025latent}. Huginn and Ouro scale recurrence to language modeling \citep{geiping2025latent,zhu2025ouro}; relaxed weight sharing \citep{bae2025relaxed}, token-specific recursion \citep{bae2025mor}, and variable-length training \citep{jeddi2026loopformer} extend its architectural flexibility. Equilibrium models \citep{bai2019deq} and Jacobian regularization \citep{yang2026stars} connect recurrence to hidden-state stability. These perspectives motivate examining what successive latent updates change in the prediction, alongside how much computation they perform.

Intermediate readouts make that evolution observable. Vocabulary-space analysis \citep{geva2022vocabulary}, tuned probes and cross-layer maps \citep{belrose2023tunedlens,yomdin2024jump} examine predictions before the final layer; layer contrasts can also guide decoding \citep{chuang2024dola}. Final-answer agreement in Ouro \citep{zhu2025ouro} and answer convergence in explicit reasoning traces \citep{liu2025answer} connect intermediate predictions to their endpoints. Early-exit methods use such signals for stopping \citep{xin2020deebert,schuster2022calm,jazbec2024fast}, while self-speculative methods verify drafts with later computation \citep{elhoushi2024layerskip,liu2024kangaroo}. The observed trajectory also depends on how an answer is scored: continuation likelihood is subject to surface-form competition \citep{holtzman2021surface}, label selection requires binding answers to symbols \citep{robinson2023multiplechoice}, and option order or label preferences can change the winner \citep{zheng2024selectors}. Our paired-format comparison connects these scoring effects to prediction dynamics across recurrent depth.

To analyze this evolution, relative comparisons provide a bridge between readout coordinates and the selected answer. Classical oscillation seminorms identify vectors differing by a constant \citep{gaubert2015dobrushin}; softmax shares that invariance \citep{gao2018softmax}, and log-ratio coordinates represent relative probability information \citep{egozcue2003logratio}. Margin-based robustness relates changes in these comparisons to decision boundaries \citep{hein2017formal,tsuzuku2018lipschitz}. Using translation-invariant comparisons and margin inequalities, we derive an exact decomposition of the gap between a magnitude-based bound and the realized endpoint margin, then examine its components across depth. For shared predictive distributions, mass and concentration give the common readout component a complementary probabilistic interpretation.

\section{The Geometry of Prediction Stability}
\label{sec:geometry}

Prediction stability depends on how score updates interact with the gaps
between candidate answers. For a fixed endpoint, we characterize answer
preservation through these relative changes and derive a hierarchy of
margin criteria, from a magnitude bound to the exact competitor constraints.
This hierarchy determines the ordering of earliest qualifying depths along
any completed score trajectory with a fixed candidate set.

\subsection{From intermediate predictions to answer comparisons}
For an input $x$, a recurrent model produces hidden states $h_t=\Phi(h_{t-1};x)$ and intermediate
predictive distributions through its output head. Fix an endpoint depth $T\ge1$
and a question with $K\ge2$ candidate answers. For scores $s_t\in\R^K$, let
$y_t=\argmax_k s_{t,k}$ under a deterministic tie rule. Endpoint agreement compares
$y_t$ with $y_T$; accuracy compares $y_t$ with the ground-truth label $y^\star$.
Candidate identities remain fixed across depths.

For a fixed intermediate depth, write $\delta=s_T-s_t$.
Write $\mathbf1=(1,\ldots,1)^\top\in\R^K$ for the all-ones vector.

The comparison between answers $k$ and $l$ changes by $\delta_k-\delta_l$.
Consequently, its largest absolute change is
\begin{equation}
 B_{\rm q}:=\max_{k,l}|\delta_k-\delta_l|
 =\osc(\delta):=\max_k\delta_k-\min_k\delta_k.
 \label{pd:eq:comparison}
\end{equation}
All comparisons are unchanged precisely on $\operatorname{span}\{\mathbf1\}$.
The equivalence class $[\delta]=\delta+\operatorname{span}\{\mathbf1\}$
therefore determines all comparison changes. The classical oscillation
seminorm \citep{gaubert2015dobrushin} induces a norm on this quotient,
measuring the largest such change.

\begin{proposition}[Comparison quotient and optimal centers]
\label{pd:prop:quotient-centers}
For $K\ge2$ and $\delta\in\R^K$, the map
$[\delta]\mapsto\osc(\delta)$ is a norm on
$\R^K/\operatorname{span}\{\mathbf1\}$.
Let $\operatorname{mid}(\delta)=(\max_k\delta_k+\min_k\delta_k)/2$.
For every $b\in\R$,
\begin{equation}
 2\|\delta-b\mathbf1\|_\infty
 =B_{\rm q}+2|\operatorname{mid}(\delta)-b|,
 \qquad
 B_{\rm q}=2\min_b\|\delta-b\mathbf1\|_\infty.
 \label{pd:eq:centering}
\end{equation}
The midrange is optimal for the maximum-coordinate norm; mean centering,
$b=\bar\delta=K^{-1}\sum_k\delta_k$, is optimal for squared Euclidean error.
Both minimizing centers are unique.
\end{proposition}
The optimized equality is the finite-dimensional quotient-norm identity
of \citet[Lemma~4.1]{gaubert2015dobrushin}.
Equation~\ref{pd:eq:centering} gives each center's exact excess over pairwise
change, including the raw radius $B_{\rm raw}=2\|\delta\|_\infty$ at $b=0$.
For comparison with conventional centers, define
\[
 b_{\rm mean}=\bar\delta,\qquad
 b_{\rm LSE}=\LSE(s_T)-\LSE(s_t),\qquad
 \LSE(v)=\log\sum_{k=1}^K e^{v_k}.
\]

\begin{proof}[Proof of Proposition~\ref{pd:prop:quotient-centers}]
Let $K\ge2$, $v\in\R^K$, and
$\alpha=\max_kv_k$, $\beta=\min_kv_k$, $q=\alpha-\beta$, $c=(\alpha+\beta)/2$.
For every $b\in\R$,
\begin{equation}
 \begin{aligned}
 \|v-b\mathbf1\|_\infty
 &=\max\{|\alpha-b|,|\beta-b|\}\\
 &=\max\{|c-b+q/2|,|c-b-q/2|\}
 =q/2+|c-b|.
 \end{aligned}
 \label{pd:eq:extrema-proof}
\end{equation}
Consequently,
\[
 \arg\min_{b\in\R}\|v-b\mathbf1\|_\infty=\{c\},\qquad
 \inf_{b\in\R}\|v-b\mathbf1\|_\infty=\frac{\osc(v)}2.
\]
For $v,w\in\R^K$ and $\lambda,b\in\R$,
\[
 \begin{aligned}
 \osc(v+b\mathbf1)&=\osc(v),&
 \osc(\lambda v)&=|\lambda|\osc(v),\\
 \osc(v+w)
 &=\max_{k,l}\{(v_k-v_l)+(w_k-w_l)\}
 \le\osc(v)+\osc(w),\\
 \osc(v)=0&\iff v\in\operatorname{span}\{\mathbf1\}.
 \end{aligned}
\]
Thus $[v]\mapsto\osc(v)$ is a norm on
$\R^K/\operatorname{span}\{\mathbf1\}$, and
\[
 \|[v]\|_{\infty,\mathrm{quot}}
 :=\inf_{b\in\R}\|v-b\mathbf1\|_\infty
 =\tfrac12\osc(v).
\]
This is the finite-dimensional specialization of
\citet[Lemma~4.1]{gaubert2015dobrushin}.

For $\bar v=K^{-1}\sum_kv_k$, orthogonality gives
\[
 \begin{aligned}
 \|v-b\mathbf1\|_2^2
 &=\|v-\bar v\mathbf1\|_2^2+K(b-\bar v)^2,\\
 \arg\min_{b\in\R}\|v-b\mathbf1\|_2^2&=\{\bar v\}.
 \end{aligned}
\]
For $K=2$, $\bar v=c$. For each $K\ge3$, the examples
\[
 \begin{array}{c|cc}
 v&2\|v\|_\infty&2\|v-\bar v\mathbf1\|_\infty\\\hline
 \mathbf1&2&0\\
 (-1,1,\ldots,1)&2&4(K-1)/K>2
 \end{array}
\]
exclude either uniform ordering of the raw and mean-centered radii.
Both are at least $\osc(v)$ by Equation~\ref{pd:eq:extrema-proof}.
\end{proof}

\subsection{From comparison changes to answer stability}
Fix a depth with a unique winner $a=y_t$. For each competitor $b\ne a$, define
its initial gap and its relative update toward the winner as
\begin{equation}
 g_b=s_{t,a}-s_{t,b}>0,\qquad
 u_b=\delta_b-\delta_a,\qquad
 m=\min_{b\ne a}g_b,\qquad h=\max_{b\ne a}u_b.
 \label{pd:eq:gap-update}
\end{equation}
The decision region for $a$ is the intersection of halfspaces
$\{\varepsilon:\varepsilon_b-\varepsilon_a<g_b\text{ for all }b\ne a\}$.
Its endpoint reserve is
\begin{equation}
 R=\min_{b\ne a}(g_b-u_b)
   =\min_{b\ne a}(s_{T,a}-s_{T,b}).
 \label{pd:eq:reserve}
\end{equation}
Thus $R>0$ characterizes $a$ as the strict endpoint winner, while $R=0$ places the
endpoint on a decision boundary. The signed quantity $h$ records the largest
competitor-relative update; $h<0$ means that every gap has increased.
Margin-based perturbation analysis relates these halfspaces to update bounds
\citep{hein2017formal,tsuzuku2018lipschitz}. The following result connects
the sharp oscillation bound to the margin attained by the actual update.

\begin{theorem}[Sharp comparison geometry and endpoint margin]
\label{pd:thm:stability}
Fix $s_t\in\R^K$, $K\ge2$, with unique winner $a$ and gaps as in
Equation~\ref{pd:eq:gap-update}. For every $r\ge0$,
\begin{equation}
 \min_{\osc(\varepsilon)\le r}\;\min_{b\ne a}
       (g_b-\varepsilon_b+\varepsilon_a)=m-r.
 \label{pd:eq:sharp-margin}
\end{equation}
Thus every allowed update preserves $a$ as a strict winner if and only if $m>r$.
For the realized update $\delta=s_T-s_t$,
\begin{equation}
 m-B_{\rm raw}\le m-B_{\rm q}\le m-h\le R,
 \label{pd:eq:ladder}
\end{equation}
and the total slack satisfies
\begin{equation}
 \begin{aligned}
 R-(m-B_{\rm raw})
 &=\underbrace{|\max_k\delta_k+\min_k\delta_k|}_{A_{\rm tr}}
  +\underbrace{(B_{\rm q}-h)}_{A_{\rm dir}}
  +\underbrace{(R-m+h)}_{A_{\rm pair}}.
 \end{aligned}
 \label{pd:eq:slack}
\end{equation}
All three terms are nonnegative, and
\begin{equation}
 A_{\rm pair}=\min_{b\ne a}\big[(g_b-m)+(h-u_b)\big].
 \label{pd:eq:pairing}
\end{equation}
Equality $A_{\rm pair}=0$ holds exactly when one competitor attains both
$g_b=m$ and $u_b=h$.
\end{theorem}
The three slacks isolate comparison-invariant translation, the use of a
symmetric radius in place of a signed winner-directed update, and extrema
attained by different competitors. Their sum is the exact gap between the
raw lower bound and the realized endpoint reserve. At $r=m$, the sharp bound
permits a tie, explaining the strict-winner condition.

\begin{proof}[Proof of Theorem~\ref{pd:thm:stability}]
Let $I=\{1,\ldots,K\}\setminus\{a\}$ and
$\mathcal B_r=\{\varepsilon\in\R^K:\osc(\varepsilon)\le r\}$.
For every $\varepsilon\in\mathcal B_r$ and $b\in I$,
\[
 g_b-\varepsilon_b+\varepsilon_a
 \ge m-\osc(\varepsilon)\ge m-r.
\]
Choose $b_\star\in\arg\min_{b\in I}g_b$ and set
$\varepsilon^\star=r e_{b_\star}$. Since $a\ne b_\star$,
\[
 \osc(\varepsilon^\star)=r,\qquad
 g_{b_\star}-\varepsilon^\star_{b_\star}+\varepsilon^\star_a=m-r.
\]
Hence the minimum is attained, and
\[
 \min_{\varepsilon\in\mathcal B_r}\min_{b\in I}
       (g_b-\varepsilon_b+\varepsilon_a)=m-r.
\]
In particular,
\[
 \bigl[\forall\varepsilon\in\mathcal B_r,\ \forall b\in I:\quad
       g_b-\varepsilon_b+\varepsilon_a>0\bigr]
 \iff m>r.
\]
For the realized update, Equation~\ref{pd:eq:extrema-proof} yields
\[
 B_{\rm raw}-B_{\rm q}
 =|\max_k\delta_k+\min_k\delta_k|\ge0.
\]
Moreover,
\[
 h=\max_{b\in I}(\delta_b-\delta_a)\le B_{\rm q},\qquad
 R=\min_{b\in I}(g_b-u_b)\ge m-h.
\]
These inequalities give Equation~\ref{pd:eq:ladder}. Its successive differences satisfy
\[
 \begin{aligned}
 R-(m-B_{\rm raw})
 &=(B_{\rm raw}-B_{\rm q})+(B_{\rm q}-h)+(R-m+h),\\
 R-m+h&=\min_{b\in I}\bigl[(g_b-m)+(h-u_b)\bigr].
 \end{aligned}
\]
For every $b\in I$, $g_b-m\ge0$ and $h-u_b\ge0$.
Since $I$ is finite and nonempty,
\[
 \begin{aligned}
 A_{\rm pair}=0
 &\iff\exists b\in I:\ (g_b-m)+(h-u_b)=0\\
 &\iff\exists b\in I:\ g_b=m\ \text{and}\ u_b=h.
 \end{aligned}
\]
\end{proof}
Since $m>0$,
\[
 m-h>0\iff m-\max\{0,h\}>0.
\]
Clipping $h$ at zero leaves the passing test unchanged; the slack identity
uses the signed $h$.

\subsection{Relating the geometry to recurrence depth}
Fix $T\in\mathbb N_{\ge1}$. For a candidate depth set $\mathcal T\subseteq\{1,\ldots,T-1\}$, let $D_j$
be the earliest depth with a unique winner at which bound $j\in\{\mathrm{raw,q,dir,res}\}$ in
Equation~\ref{pd:eq:ladder} is positive; set $D_j=T$ if none qualifies.
Write $F_j(d)=\Pr(D_j\le d)$ for its cumulative earliest-depth curve.

\begin{corollary}[Nested earliest-depth distributions]
\label{pd:cor:depth}
On the same candidate set, the four tests in
Theorem~\ref{pd:thm:stability} give
$D_{\rm raw}\ge D_{\rm q}\ge D_{\rm dir}\ge D_{\rm res}$ on every trajectory.
Consequently $F_{\rm raw}\le F_{\rm q}\le F_{\rm dir}\le F_{\rm res}$, and
the normalized depth area satisfies
\begin{equation}
 W_j=1-\frac{\Ee[D_j]}{T}
     =\frac1T\sum_{d=1}^{T-1}F_j(d).
 \label{pd:eq:depth-area}
\end{equation}
\end{corollary}
Expectations are over complete questions, without independence across depths.
Differences in $W$ are areas between the curves. The proof also gives the
weighted-sum form for sparse candidate sets.

\begin{proof}[Proof of Corollary~\ref{pd:cor:depth}]
Fix $T\in\mathbb N_{\ge1}$ and
$\mathcal T\subseteq\{1,\ldots,T-1\}$. For each trajectory, define
\[
 \mathcal U=\{t\in\mathcal T:s_t\text{ has a unique maximizer}\}.
\]
For $t\in\mathcal U$, let $L_j(t)$ denote the corresponding bound in
Equation~\ref{pd:eq:ladder}, and set
\[
 \mathcal A_j=\{t\in\mathcal U:L_j(t)>0\},\qquad
 D_j=\min(\mathcal A_j\cup\{T\}).
\]
Theorem~\ref{pd:thm:stability} implies
\[
 \mathcal A_{\rm raw}\subseteq\mathcal A_{\rm q}
 \subseteq\mathcal A_{\rm dir}\subseteq\mathcal A_{\rm res}
 \quad\Longrightarrow\quad
 D_{\rm raw}\ge D_{\rm q}\ge D_{\rm dir}\ge D_{\rm res}.
\]
For every $d\in\R$,
\[
 \{D_{\rm raw}\le d\}\subseteq\{D_{\rm q}\le d\}
 \subseteq\{D_{\rm dir}\le d\}\subseteq\{D_{\rm res}\le d\}.
\]
Taking probabilities gives the CDF ordering. Since $1\le D_j\le T$,
\begin{equation}
 T-D_j=\sum_{d=1}^{T-1}\mathbf1\{D_j\le d\}.
 \label{pd:eq:depth-indicator}
\end{equation}
Linearity of expectation therefore gives
\[
 1-\frac{\Ee[D_j]}T
 =\frac1T\sum_{d=1}^{T-1}\Pr(D_j\le d)
 =\frac1T\sum_{d=1}^{T-1}F_j(d).
\]
For $\mathcal T=\{t_1<\cdots<t_J\}\ne\varnothing$, set $t_{J+1}=T$.
Since $D_j\in\mathcal T\cup\{T\}$,
\[
 T-D_j=\sum_{k=1}^J(t_{k+1}-t_k)\mathbf1\{D_j\le t_k\}.
\]
Taking expectations yields
\begin{equation}
 W_j=\frac1T\sum_{k=1}^J(t_{k+1}-t_k)F_j(t_k).
 \label{pd:eq:sparse-area}
\end{equation}
If $\mathcal T=\varnothing$, then $D_j=T$ and $W_j=0$, consistent with the
empty-sum convention.
\end{proof}
The terminal value $F_j(T)=1$ includes the fallback. The depth areas compare
earliest qualifying depths on a chosen grid; the slacks in
Equation~\ref{pd:eq:slack} concern a single interval.

For $j\in\{\mathrm{mean,LSE}\}$, define $D_j$ on the same candidate set by
\[
 D_j=\min\bigl(\{t\in\mathcal T:
    s_t\text{ has a unique winner},\ 
    m_t>2\|s_T-s_t-b_{j,t}\mathbf1\|_\infty\}\cup\{T\}\bigr),
\]
where $m_t=\min_{k\ne y_t}(s_{t,y_t}-s_{t,k})$ and $b_{j,t}$ is the corresponding center for $s_T-s_t$.
We use the same definition of $W_j$ and the trajectorywise reference
$D_{\rm env}=\min(D_{\rm mean},D_{\rm LSE})$, with $W_{\rm env}=1-\Ee[D_{\rm env}]/T$.
This reference selects the earlier qualifying depth separately for each
completed trajectory.

\FloatBarrier
\section{Finite-Depth Prediction Geometry}
\label{pd:sec:findings}

The criteria in Section~\ref{sec:geometry} can now be evaluated on completed
trajectories. We begin with complete answer-text scoring, then examine
matched-label trajectories and the broader archived centering comparisons.
A coordinate-permutation reference tests how the observed updates align with
candidate identities, and answer-transition counts distinguish endpoint
agreement from accuracy against ground truth.

Our main model comparisons use Huginn-3.5B and Ouro-1.4B on ARC-Challenge
\citep{clark2018arc} and MMLU \citep{hendrycks2021mmlu}. Each result reports
its scoring condition and question population. Model configurations, data
selection and resampling procedures are collected in
Appendix~\ref{pd:sec:observation}.

\paragraph{Score constructions.}
\label{arxiv:sec:scoring}
In \emph{label scoring} (F1), $v_k$ is candidate $k$'s distinct label token and all
scores come from one predictive distribution:
\begin{equation}
 s^{\mathrm{lab}}_{t,k}=\log p_t(v_k\mid x).
 \label{pd:eq:f1}
\end{equation}
In \emph{answer-text scoring}, candidate $k$ has $n_k$ continuation tokens
$v_{k,1:n_k}$ and character length $\ell_k$, including its leading space:
\begin{equation}
 s^{\mathrm{txt}}_{t,k}=\frac1{\ell_k}\sum_{j=1}^{n_k}
 \log p_t(v_{k,j}\mid x,v_{k,<j}).
 \label{pd:eq:f2}
\end{equation}
Here $x$ is the rendered prompt for the condition being evaluated. Label scoring
displays all answer options. We evaluate answer-text scoring both without the
options block (F2) and with it (F3); F3 uses four balanced arrangements and maps
predictions back to candidate identity. Label/F2 comparisons therefore change
the prompt and continuation execution as well as the scoring formula. Within F3,
every geometric criterion is evaluated on the same execution. The geometric
analysis applies to all these real-valued score vectors; the shared-distribution
probability identities apply to the label scores in Equation~\ref{pd:eq:f1}.

\paragraph{Empirical depth-area decomposition.}
Translation removal admits endpoint-preserving depths whenever
$B_{\rm q}<m\le B_{\rm raw}$. For $N$ complete trajectories, its average
contribution to normalized depth area is
\begin{equation}
 \Delta_{\rm tr}=W_{\rm q}-W_{\rm raw}
 =\frac1{NT}\sum_{i=1}^N(D_{{\rm raw},i}-D_{{\rm q},i}).
 \label{pd:eq:promotion}
\end{equation}
The nested tests decompose the empirical depth area:
\begin{equation}
 W_{\rm res}-W_{\rm raw}
 =\underbrace{(W_{\rm q}-W_{\rm raw})}_{\Delta_{\rm tr}}
 +\underbrace{(W_{\rm dir}-W_{\rm q})}_{\Delta_{\rm dir}}
 +\underbrace{(W_{\rm res}-W_{\rm dir})}_{\Delta_{\rm pair}}.
 \label{r3:eq:depth-components}
\end{equation}
Nesting fixes the signs; magnitudes depend on the joint distribution of margins and updates. At a unique-winner depth, the two relaxations add admissible cases precisely when
\[
 \begin{aligned}
 \text{direction:}&\quad h<m\le B_{\rm q},\\
 \text{pairing:}&\quad m\le h,\qquad u_b<g_b\quad\text{for all }b\ne a.
 \end{aligned}
\]
The signed test includes favorable updates $h<0$, even for $K=2$; pairing
retains each competitor's individual constraint.

\subsection{Geometry under full answer-text scoring}
\label{pd:app:text-geometry}

Full answer-text scoring with all options displayed gives substantial
contributions from both direction and pairing across the four model/task
cells: $\Delta_{\rm dir}=9.77$--$22.07$ percentage points (pp) and
$\Delta_{\rm pair}=10.94$--$12.70$ pp. Their combined contribution beyond
translation removal is $22.46$--$34.38$ pp (Figure~\ref{pd:fig:text-geometry}).
Each task has 32 previously inspected questions, reused from the paired-format
study described in Appendix~\ref{pd:app:populations}. Four balanced arrangements
are averaged within each question before averaging questions. Every method uses
each arrangement's own endpoint and terminal fallback, so paired increments
compare the same executions. These are complete-population geometric comparisons;
correctness is not used to select trajectories.

\begin{figure}[!htbp]
\centering
\includegraphics[width=\linewidth]{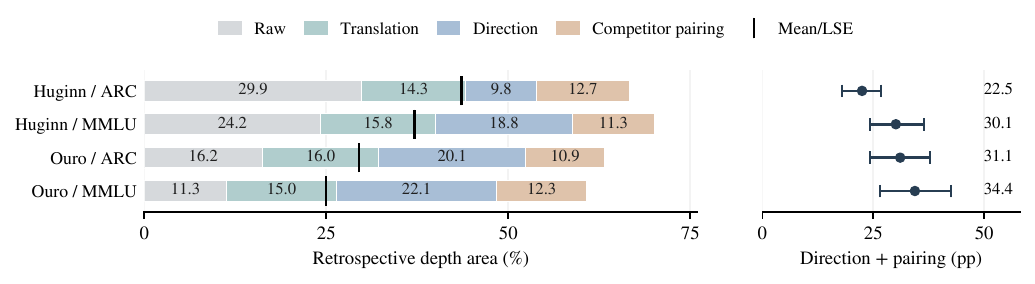}
\caption{Direction and competitor pairing account for substantial conservatism beyond translation. Full answer-text scoring, 32 paired questions per task, four option arrangements. Left: raw depth area and the three increments in Equation~\ref{r3:eq:depth-components}; black marks show the mean/LSE reference. Right: the combined direction and pairing contribution, with 95\% whole-question bootstrap intervals. Candidate depths are $T/4,T/2,3T/4$ with endpoint $T=32$ for Huginn and $T=4$ for Ouro; $D=T$ if no earlier depth qualifies.}
\label{pd:fig:text-geometry}
\end{figure}

Table~\ref{pd:app:tab:text-geometry} gives the per-cell estimates and paired
intervals on the quarter-depth grid. The mean/LSE envelope chooses the earlier
passing depth for each completed trajectory. Here LSE is an algebraic center
of four real-valued text scores; the shared-distribution mass identity applies
only to the label scores in F1.

\begin{table}[t]
\centering\small\setlength{\tabcolsep}{4pt}
\caption{Winner-aware geometry under full answer-text scoring with all options displayed. The same 32 questions per task are evaluated under four balanced arrangements in each model. Increments are retrospective depth-area percentage points on the quarter-depth grid, with marginal 95\% paired question-bootstrap intervals. Native horizons are 32 for Huginn and 4 for Ouro.}
\label{pd:app:tab:text-geometry}
\begin{tabular}{lccc}
\toprule
Model / task & Direction & Competitor pairing & Reserve--quotient\\
\midrule
Huginn / ARC & $9.77\,[7.03,12.50]$ & $12.70\,[9.57,16.02]$ & $22.46\,[17.97,26.76]$ \\
Huginn / MMLU & $18.75\,[13.67,24.41]$ & $11.33\,[6.84,16.21]$ & $30.08\,[24.22,36.33]$ \\
Ouro / ARC & $20.12\,[13.87,27.15]$ & $10.94\,[6.25,15.82]$ & $31.05\,[24.22,37.70]$ \\
Ouro / MMLU & $22.07\,[15.23,29.30]$ & $12.30\,[5.86,19.92]$ & $34.38\,[26.56,42.58]$ \\
\bottomrule
\end{tabular}
\end{table}

Conventional centering already captures most of the translation-related
improvement in these text trajectories: $W_{\rm q}-W_{\rm env}=0.59$--$2.93$
pp. The quotient-over-envelope increments are $0.59\,[0.00,1.56]$ and
$2.93\,[1.37,4.49]$ pp for Huginn ARC/MMLU, and $2.73\,[0.98,5.08]$ and
$1.37\,[0.20,2.93]$ pp for Ouro. The larger remaining gap is resolved by the
direction of relative updates and their pairing with individual gaps.

The comparison also extends across arithmetic settings and observation grids.
With native output-head arithmetic, the combined direction and pairing
increments are 22.46/29.69 pp for Huginn and 31.45/34.38 pp for Ouro. On
Huginn's full native candidate grid $1{:}31$, they are 24.58/23.90 pp;
Ouro's quarter grid already includes every preterminal depth.
Figure~\ref{pd:app:fig:text-cdf} shows the native-grid distributions.

\begin{figure}[!htbp]
\centering
\includegraphics[width=\linewidth]{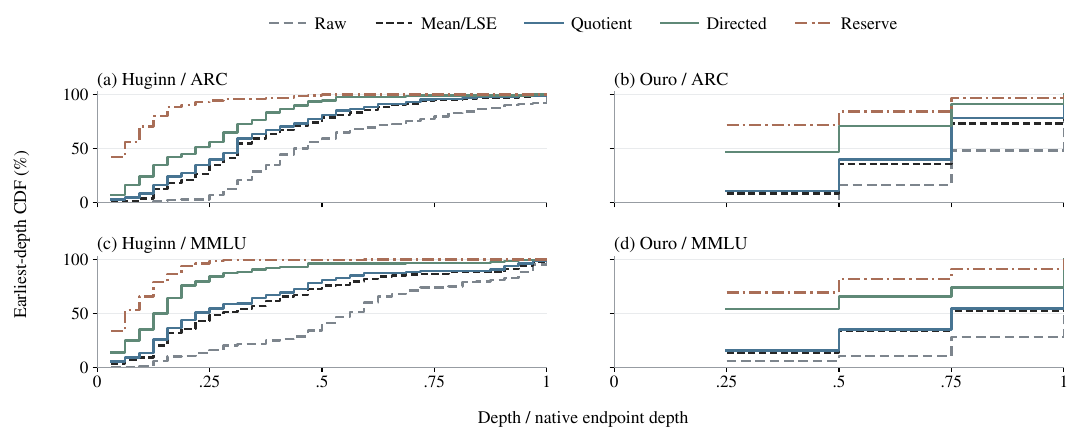}
\caption{Earliest-depth distributions under full answer-text scoring. Curves use each model's native candidate grid and include the shared terminal fallback. Each cell averages four arrangements within each of 32 questions. The decomposition in Figure~\ref{pd:fig:text-geometry} uses the quarter grid for both models.}
\label{pd:app:fig:text-cdf}
\end{figure}

\FloatBarrier
\subsection{Matched-label trajectories}
\label{pd:app:label-geometry}
\label{pd:app:direction-transitions}
\label{pd:app:distributions}

The matched-label collection supplies a shared-distribution readout on a
separate population of $128$ questions per task, absent from the preceding
question manifests. On its four model/task cells,
with $\mathcal T=\{T/4,T/2,3T/4\}$,
$\Delta_{\rm dir}=8.98$--$22.66$ pp and
$\Delta_{\rm pair}=4.30$--$10.74$ pp. Table~\ref{pd:app:tab:levels} reports
the absolute depth areas, and Table~\ref{pd:app:tab:direction} and
Figure~\ref{pd:app:fig:label-increments} give the paired decomposition.

\begin{table}[!htbp]
\centering\small
\caption{Complete normalized depth areas for the four matched-label
model/task conditions, in percent. Candidates are $T/4,T/2,3T/4$;
$N=128$ per task. Passing conditions have zero endpoint disagreements.}
\label{pd:app:tab:levels}
\begin{tabular}{llrrrrrr}
\toprule
Model & Task & $T$ & Raw & Mean & Quotient & Directed & Reserve\\
\midrule
Huginn-3.5B & ARC & 32 & 22.07 & 29.69 & 32.62 & 43.55 & 54.30\\
 & MMLU & 32 & 32.23 & 37.89 & 42.19 & 51.17 & 61.72\\
Ouro-1.4B & ARC & 4 & 31.84 & 35.16 & 37.30 & 51.76 & 56.05\\
 & MMLU & 4 & 21.29 & 23.63 & 26.17 & 48.83 & 55.27\\
\bottomrule
\end{tabular}
\end{table}

\begin{table}[!htbp]
\centering\small
\caption{Paired increases in retrospective depth area, in percentage points,
from translation, direction and competitor pairing. All four population conditions use
$128$ questions; intervals are paired $95\%$ whole-question intervals.}
\label{pd:app:tab:direction}
\begin{tabular}{llrrr}
\toprule
Model & Task & $W_{\rm q}-W_{\rm raw}$ & $W_{\rm dir}-W_{\rm q}$ & $W_{\rm res}-W_{\rm dir}$\\
\midrule
Huginn-3.5B & MMLU & $9.96\,[6.64,13.48]$ & $8.98\,[6.64,11.52]$ & $10.55\,[7.42,13.87]$\\
Huginn-3.5B & ARC & $10.55\,[7.81,13.28]$ & $10.94\,[8.20,13.67]$ & $10.74\,[7.42,14.26]$\\
Ouro-1.4B & MMLU & $4.88\,[3.12,6.84]$ & $22.66\,[19.14,26.37]$ & $6.45\,[3.91,9.38]$\\
Ouro-1.4B & ARC & $5.47\,[3.52,7.62]$ & $14.45\,[11.72,17.19]$ & $4.30\,[2.34,6.64]$\\
\bottomrule
\end{tabular}
\end{table}

The quotient also improves on mean centering in every cell. The paired gains
are $2.93\,[1.56,4.30]$ pp on Huginn ARC,
$4.30\,[2.34,6.64]$ pp on Huginn MMLU,
$2.15\,[0.98,3.52]$ pp on Ouro ARC, and
$2.54\,[1.17,4.30]$ pp on Ouro MMLU.

\begin{figure}[!htbp]
\centering
\includegraphics[width=\linewidth]{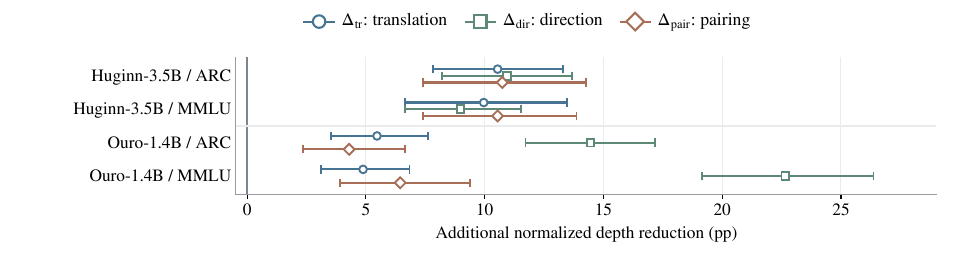}
\caption{Direction and competitor pairing on the matched-label collections. The increments in Equation~\ref{r3:eq:depth-components} on $\{T/4,T/2,3T/4\}$, with paired 95\% whole-question intervals; 128 F1 questions per model and task. Figure~\ref{pd:fig:ladder} locates the additional qualifying questions across the full native grids.}
\label{pd:app:fig:label-increments}
\end{figure}

\FloatBarrier
To locate the gains across recurrent depth, let
$F_j(d)=N^{-1}\sum_i\mathbf1\{D_{j,i}\le d\}$ on the native grid,
with $N=128$ for each label-scored cell. Nesting gives two disjoint sets of
additional questions that have qualified by depth $d$:
\begin{equation}
 \begin{aligned}
 G_{\rm tr}(d)&=F_{\rm q}(d)-F_{\rm raw}(d)
 =\frac1N\#\{i:D_{{\rm q},i}\le d<D_{{\rm raw},i}\},\\
 G_{\rm dp}(d)&=F_{\rm res}(d)-F_{\rm q}(d)
 =\frac1N\#\{i:D_{{\rm res},i}\le d<D_{{\rm q},i}\}.
 \end{aligned}
 \label{arxiv:eq:depth-gains}
\end{equation}
Their sum is $F_{\rm res}(d)-F_{\rm raw}(d)$. On the full integer grid,
$T^{-1}\sum_{d=1}^{T-1}G_{\rm tr}(d)=W_{\rm q}-W_{\rm raw}$ and
$T^{-1}\sum_{d=1}^{T-1}G_{\rm dp}(d)=W_{\rm res}-W_{\rm q}$.
Figure~\ref{pd:fig:ladder} separates these contributions at each depth.
Unlike the cumulative distributions in Figure~\ref{pd:app:fig:text-cdf},
these gain profiles can rise or fall as the stricter criteria catch up;
both vanish at the shared terminal fallback.

\begin{figure}[!htbp]
\centering
\includegraphics[width=\linewidth]{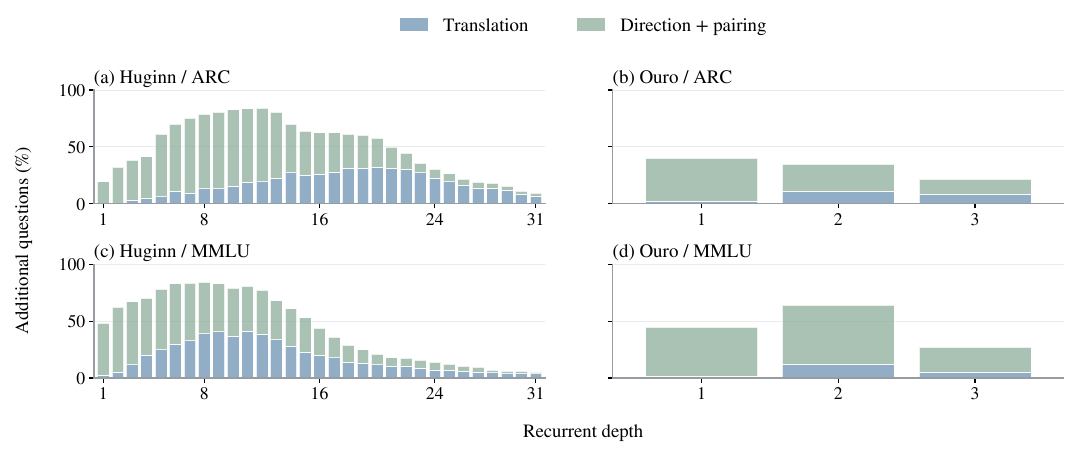}
\caption{Where the geometric gains arise under label scoring. At each preterminal native depth $d$, blue gives the additional fraction of questions that have qualified after translation removal, $G_{\rm tr}(d)$; green adds the contribution of direction and pairing, $G_{\rm dp}(d)$. The stacked height is the reserve-over-raw gain by that depth. Each panel uses 128 F1 questions; $T=32$ for Huginn and $T=4$ for Ouro. Both components are zero at $T$, which is omitted. These profiles locate the gains, whereas Figure~\ref{pd:app:fig:text-cdf} reports complete earliest-depth distributions for the separate answer-text population.}
\label{pd:fig:ladder}
\label{pd:app:fig:arc-cdf}
\end{figure}

\Needspace{8\baselineskip}
\subsection{Translation and centering on archived trajectories}
\label{arxiv:sec:archived-centering}

The archived trajectories provide a larger population and multiple endpoints
for examining translation removal and conventional centers. These previously
inspected Huginn and Recurrent-Llama data retain the native evaluator's scoring
conventions: ARC uses continuation log likelihood normalized by character
length, while each MMLU continuation contains one answer token. The complete
population and configuration are specified in Appendix~\ref{pd:app:populations}.
On Huginn's $N=2{,}704$ native trajectories with endpoint $T=32$ and
candidates $4{:}31$, $W_{\rm raw}=42.65\%$ and $W_{\rm q}=55.28\%$,
giving $\Delta_{\rm tr}=12.63\,[12.01,13.31]$ pp.
The fraction qualifying before the endpoint changes from $88.76\%$ to
$91.46\%$. These statistics use endpoint displacements $s_T-s_t$.
To distinguish local motion from endpoint displacement, define
\[
 \delta_{i,t}^{(T)}=s_{i,T}-s_{i,t},\qquad
 d_{i,t}=s_{i,t+1}-s_{i,t},\qquad
 \delta_{i,t}^{(T)}=\sum_{u=t}^{T-1}d_{i,u}.
\]
Let $\mathcal J$ index the observed adjacent pairs in one archived collection,
and set $\mathcal I=\{(i,t)\in\mathcal J:\|d_{i,t}\|_\infty>0\}$.
Assume $\mathcal I\ne\varnothing$. For $(i,t)\in\mathcal I$, define
\[
 r_{i,t}=2\|d_{i,t}\|_\infty,\qquad
 q_{i,t}=\osc(d_{i,t}),\qquad
 \rho_{i,t}=\frac{q_{i,t}}{r_{i,t}}\in[0,1].
\]
Equation~\ref{pd:eq:centering} gives the exact identities
\begin{equation}
 \begin{aligned}
 r_{i,t}-q_{i,t}&=2|\operatorname{mid}(d_{i,t})|,\\
 \tau&:=\frac{\sum_{\mathcal I}(r_{i,t}-q_{i,t})}
                   {\sum_{\mathcal I}r_{i,t}}
 =1-\frac{\sum_{\mathcal I}q_{i,t}}{\sum_{\mathcal I}r_{i,t}}
 =\sum_{\mathcal I}w_{i,t}(1-\rho_{i,t})\in[0,1],\\
 w_{i,t}&:=\frac{r_{i,t}}{\sum_{\mathcal I}r_{i,t}},\qquad
 \sum_{\mathcal I}w_{i,t}=1,\\
 \rho_{i,t}=0&\iff d_{i,t}\in\operatorname{span}\{\mathbf1\}\setminus\{0\}.
 \end{aligned}
 \label{pd:eq:tau}
\end{equation}
Here $\sum_{\mathcal I}$ sums over $(i,t)\in\mathcal I$; zero increments
contribute neither radius and are omitted to avoid an undefined ratio.
Thus $\tau$ measures the removable fraction of the aggregate adjacent-update
radius, rather than the unweighted frequency of pure translations or a
depth-area gain. In particular,
$\osc(\delta_{i,t}^{(T)})\le\sum_{u=t}^{T-1}\osc(d_{i,u})$ need not be an
equality, so adjacent-update ratios cannot substitute for endpoint tests.
The aggregate raw-radius reductions are $44.76\%$ in Huginn and
$28.41\%$ in Recurrent-Llama. Substantial translation-removable radius remains
under FP32 arithmetic, while the exact zero-contrast increments observed in
BF16 disappear (Section~\ref{pd:app:extended-precision}).

For an endpoint $T$ and candidate set $\mathcal T$, let
$I_t=\{i:s_{i,t}\text{ has a unique maximizer}\}$ and define
\[
 B_{{\rm raw},i,t}=2\|\delta_{i,t}^{(T)}\|_\infty,\qquad
 B_{{\rm q},i,t}=\osc(\delta_{i,t}^{(T)}),\qquad
 \mathcal A_{j,t}=\{i\in I_t:B_{j,i,t}<m_{i,t}\}
 \quad(j\in\{\mathrm{raw,q}\}).
\]
Since $\mathcal A_{{\rm raw},t}\subseteq\mathcal A_{{\rm q},t}$, the
pointwise promotion set is
\begin{equation}
 \mathcal P_t:=\mathcal A_{{\rm q},t}\setminus\mathcal A_{{\rm raw},t}
 =\{i\in I_t:B_{{\rm q},i,t}<m_{i,t}\le B_{{\rm raw},i,t}\}.
 \label{pd:app:eq:promotion-set}
\end{equation}
With $D_{j,i}=\min(\{t\in\mathcal T:i\in\mathcal A_{j,t}\}\cup\{T\})$,
\[
 D_{{\rm q},i}<D_{{\rm raw},i}
 \iff\exists t\in\mathcal T:\ i\in\mathcal P_t,\quad t<D_{{\rm raw},i}.
\]
Indeed, the forward implication takes $t=D_{{\rm q},i}$; the reverse
implication follows from $D_{{\rm q},i}\le t<D_{{\rm raw},i}$.
For Huginn ($N=2{,}704$, $T=32$), $|\mathcal P_{16}|=346$.
The depth-area change in Equation~\ref{pd:eq:promotion} averages the advance
in earliest passing depth over complete trajectories; this pointwise count
alone does not determine that area.

Mean centering gives $W_{\rm mean}=52.25\%$ for archived Huginn at $T=32$.
The archived trajectories provide the multi-endpoint centering comparison in
Table~\ref{pd:tab:centers}. At each $T\in\{24,28,32\}$, all centers use the
same candidates $4{:}T-1$, endpoint displacement $s_T-s_t$, and terminal
fallback $D=T$. For $c\in\{\mathrm{mean,LSE}\}$, define
\[
 \begin{aligned}
 b_{\mathrm{mean},i,t}^{(T)}&=K^{-1}\mathbf1^\top\delta_{i,t}^{(T)},&
 b_{\mathrm{LSE},i,t}^{(T)}&=\LSE(s_{i,T})-\LSE(s_{i,t}),\\
 B_{c,i,t}&=2\|\delta_{i,t}^{(T)}-b_{c,i,t}^{(T)}\mathbf1\|_\infty,&
 D_{c,i}&=\min\bigl(\{t\in\mathcal T:i\in I_t,\ B_{c,i,t}<m_{i,t}\}\cup\{T\}\bigr).
 \end{aligned}
\]
The per-trajectory envelope and its empirical depth area are
\[
 D_{\mathrm{env},i}:=\min\{D_{\mathrm{mean},i},D_{\mathrm{LSE},i}\},\qquad
 W_{\mathrm{env}}:=1-\frac1{NT}\sum_{i=1}^N D_{\mathrm{env},i}.
\]
Optimal centering implies
\[
 D_{{\rm q},i}\le D_{\mathrm{env},i},\qquad
 W_{\rm q}\ge W_{\mathrm{env}}\ge\max\{W_{\mathrm{mean}},W_{\mathrm{LSE}}\}.
\]
The envelope therefore compares against the better center separately on each
completed trajectory, which can be stronger than choosing one center for the
whole population. Its selection is retrospective.

\begin{table}[!htbp]
\centering\small\setlength{\tabcolsep}{4.4pt}
\caption{Endpoint geometry beyond conventional centers. Paired increases in normalized depth area $W$ (pp, 95\% intervals), using all 2,704 native trajectories per checkpoint and candidates $4{:}T-1$. The envelope chooses the earlier mean/LSE result separately for each trajectory.}
\label{pd:tab:centers}
\begin{tabular}{llccc}
\toprule
Model & $T$ & Quotient $-$ mean & Quotient $-$ LSE & Quotient $-$ envelope\\
\midrule
Huginn-3.5B & 24 & $3.28\,[2.89,3.66]$ & $3.23\,[2.86,3.62]$ & $2.34\,[2.03,2.67]$\\
 & 28 & $3.58\,[3.18,4.00]$ & $3.30\,[2.92,3.70]$ & $2.62\,[2.28,2.98]$\\
 & 32 & $3.03\,[2.67,3.40]$ & $2.88\,[2.54,3.24]$ & $2.14\,[1.86,2.45]$\\
\midrule
RL-1.4B (R32) & 24 & $0.86\,[0.70,1.04]$ & $1.05\,[0.89,1.23]$ & $0.62\,[0.49,0.77]$\\
 & 28 & $0.82\,[0.68,0.97]$ & $0.88\,[0.76,1.00]$ & $0.54\,[0.45,0.64]$\\
 & 32 & $0.77\,[0.62,0.93]$ & $1.03\,[0.85,1.21]$ & $0.60\,[0.47,0.74]$\\
\bottomrule
\end{tabular}
\end{table}

At $T=32$, the quotient-over-envelope gains are
$2.14\,[1.86,2.45]$ pp for Huginn and $0.60\,[0.47,0.74]$ pp for
Recurrent-Llama.

\FloatBarrier
\subsection{Observed alignment and coordinate permutations}
\label{arxiv:sec:coordinate-permutations}

The nonnegative depth increments follow from nested criteria; they do not
establish that the model's updates align more favorably with candidate gaps
than alternative assignments of the same update coordinates. To examine
alignment beyond these algebraic inequalities, hold $s_t$ fixed and enumerate
all $24$ coordinate permutations of each four-option $\delta$.
For $K=4$, the exact coordinate-permutation reference is
\begin{equation}
 p_{\rm perm}(s_t,\delta)=\frac1{24}\sum_{P\in\mathfrak S_4}
       \mathbf1\{R_t(P\delta)>0\},\qquad
 R_t(\varepsilon)=\min_{b\ne a}[g_b-(\varepsilon_b-\varepsilon_a)].
 \label{r3:eq:permutation}
\end{equation}
This preserves the update's mean, span, total energy, and contrast energy.
The reference averages strict preservation over the resulting score vectors;
observed preservation is then compared with it on every positive-margin
question-depth pair.

\begin{table}[!htbp]
\centering\small
\caption{Actual preservation versus the exact coordinate-permutation
reference. Values average the three quarter-grid depths within each question,
then $128$ questions. Differences are percentage points.}
\label{pd:app:tab:permutation}
\begin{tabular}{llrrr}
\toprule
Model & Task & Actual (\%) & Reference (\%) & Difference [95\% interval]\\
\midrule
Huginn-3.5B & MMLU & $78.12$ & $82.34$ & $-4.21\,[-7.35,-1.12]$\\
Huginn-3.5B & ARC & $69.53$ & $76.53$ & $-7.00\,[-10.80,-3.14]$\\
Ouro-1.4B & MMLU & $68.75$ & $66.56$ & $2.19\,[-1.94,6.27]$\\
Ouro-1.4B & ARC & $73.18$ & $74.08$ & $-0.90\,[-4.29,2.21]$\\
\bottomrule
\end{tabular}
\end{table}

Table~\ref{pd:app:tab:permutation} shows lower preservation in Huginn than
this magnitude-matched reference and unresolved differences in Ouro. The
reference acts on score coordinates; it does not impose vocabulary
normalization or the model's transition law. Independent permutations at
different depths likewise supply no consistent recurrent trajectory.

\FloatBarrier
\subsection{Endpoint agreement and correctness}
\label{arxiv:sec:answer-correction}

The preceding comparisons concern agreement with the fixed endpoint.
Ground-truth labels distinguish the different consequences of answer changes.
Let $A_t=N^{-1}\sum_i\mathbf1\{y_{i,t}=y_i^\star\}$. For transitions $t\to T$, let $N_{\rm W\to C}$, $N_{\rm C\to W}$ and $N_{\rm W\to W'}$ count corrections, degradations and switches between distinct wrong answers, respectively. Then
\begin{equation}
 \begin{aligned}
 A_T-A_t&=(N_{\rm W\to C}-N_{\rm C\to W})/N,\\
 \frac1N\sum_i\mathbf1\{y_{i,t}\ne y_{i,T}\}
  &=(N_{\rm W\to C}+N_{\rm C\to W}+N_{\rm W\to W'})/N.
 \end{aligned}
 \label{r3:eq:correctness}
\end{equation}
Endpoint disagreement therefore includes correction, degradation, and
switches between wrong answers. Stability and correctness describe distinct
aspects of the same trajectory.

Under full answer-text scoring, endpoint accuracy averaged over all
arrangements is 49.22/43.75\% for Huginn ARC/MMLU and 65.62/49.22\% for Ouro.
Table~\ref{pd:app:tab:text-accuracy} gives the corresponding 95\%
question-bootstrap intervals. As in the geometric comparisons, every
question and arrangement contributes to these estimates.

\begin{table}[!htbp]
\centering\small
\caption{Endpoint accuracy under all-options answer-text scoring, averaged over four arrangements per question. Values are percentages with 95\% whole-question intervals; $N=32$ per task.}
\label{pd:app:tab:text-accuracy}
\begin{tabular}{llc}
\toprule
Model & Task & Accuracy [95\% interval]\\
\midrule
Huginn & ARC & $49.22\,[35.16,62.52]$ \\
Huginn & MMLU & $43.75\,[28.12,59.38]$ \\
Ouro & ARC & $65.62\,[50.00,79.69]$ \\
Ouro & MMLU & $49.22\,[32.81,65.62]$ \\
\bottomrule
\end{tabular}
\end{table}

For the matched-label collection, label indices and choice order were matched
to the original source records. Huginn's F1 endpoint accuracy is $44/128$ on
MMLU and $27/128$ on ARC. Table~\ref{pd:app:tab:transitions} separates the
start-to-end transitions on all $N=128$ questions per task. At Huginn MMLU
$t=4$, $(N_{\rm W\to C},N_{\rm C\to W},N_{\rm W\to W'})=(24,7,30)$:
the net accuracy gain is $13.28$ pp, while 30 answer changes leave the answer
incorrect. Early MMLU contains a net correction component, while
wrong-to-wrong switches remain numerous. Across the three quarter-grid depths,
excess correction, degradation, and wrong-to-wrong frequencies relative to
permutation are $2.04/0.37/1.80$ pp for MMLU and $1.73/1.68/3.59$ pp for ARC.

\begin{table}[!htbp]
\centering\footnotesize
\caption{Huginn-3.5B (0125) F1 start-to-end answer transitions on $128$ questions per
task. W$\to$W denotes a switch to a different wrong answer; W$=$W denotes
the same wrong answer. The five categories sum to $128$ in each row.
Accuracy differences are percentage points.}
\label{pd:app:tab:transitions}
\setlength{\tabcolsep}{3pt}
\begin{tabular}{lrrrrrrr}
\toprule
Task & $t$ & W$\to$C & C$\to$W & W$\to$W & C$\to$C & W$=$W & Accuracy change [95\% interval]\\
\midrule
MMLU & 1 & 25 & 10 & 31 & 19 & 43 & $11.72\,[3.12,20.31]$\\
 & 4 & 24 & 7 & 30 & 20 & 47 & $13.28\,[5.47,21.88]$\\
 & 8 & 18 & 12 & 26 & 26 & 46 & $4.69\,[-3.91,13.28]$\\
 & 16 & 4 & 6 & 10 & 40 & 68 & $-1.56\,[-6.25,3.12]$\\
 & 24 & 1 & 2 & 5 & 43 & 77 & $-0.78\,[-3.12,1.56]$\\
 & 31 & 0 & 1 & 3 & 44 & 80 & $-0.78\,[-2.34,0.00]$\\
\midrule
ARC & 1 & 22 & 23 & 58 & 5 & 20 & $-0.78\,[-11.72,9.38]$\\
 & 4 & 23 & 27 & 51 & 4 & 23 & $-3.12\,[-14.06,7.81]$\\
 & 8 & 21 & 12 & 44 & 6 & 45 & $7.03\,[-1.56,15.62]$\\
 & 16 & 9 & 10 & 10 & 18 & 81 & $-0.78\,[-7.03,6.25]$\\
 & 24 & 3 & 3 & 5 & 24 & 93 & $0.00\,[-3.91,3.91]$\\
 & 31 & 2 & 2 & 1 & 25 & 98 & $0.00\,[-3.12,3.12]$\\
\bottomrule
\end{tabular}
\end{table}

\FloatBarrier

\FloatBarrier
\section{Composition of Score Updates}
\label{pd:sec:readout}

The preceding margin analysis relates score changes to the decision boundary.
Here we hold the scoring construction fixed and examine the composition of
those changes. An orthogonal decomposition separates common and contrast
energy; for scores from a shared predictive distribution, mass and
concentration further resolve the common component. We then use this
decomposition to distinguish changes in absolute energy, relative allocation,
and covariance across questions.

\subsection{Mass and concentration in a shared distribution}
\label{pd:app:mass-interpretation}

Fix a set $\mathcal C$ of $K$ distinct candidate tokens with positive probabilities.
For an interval $(r,t)$, write $\Delta f=f_t-f_r$ and $\delta=\Delta s$.
Let
\[
 M_t=\sum_{k\in\mathcal C}p_{t,k},\qquad
 \pi_{t,k}=p_{t,k}/M_t,\qquad
 \kappa_t=D_{\rm KL}(U_K\Vert\pi_t),\qquad
 P_\perp=I-\mathbf1\mathbf1^\top/K.
\]
Here $M_t$ is the total probability mass assigned to the candidate set,
$\pi_t$ is the conditional distribution within that set, and $U_K$ is the
uniform distribution on its $K$ candidates. The reverse KL divergence
$\kappa_t$ measures within-set concentration relative to uniformity;
$P_\perp$ projects onto the contrast subspace $\mathbf1^\perp$.
The common Euclidean projection uses $\bar\delta$; the optimal
$L_\infty$ center in Equation~\ref{pd:eq:centering} uses the midrange.

\begin{proposition}[Mass and concentration in a shared readout]
\label{pd:prop:readout}
For selected log probabilities as in Equation~\ref{pd:eq:f1},
\begin{equation}
 s_t=\log M_t\,\mathbf1+\log\pi_t.
 \label{pd:eq:mass}
\end{equation}
Across any two depths, the common and contrast components satisfy
\begin{equation}
 c:=\bar\delta=\Delta\log M-\Delta\kappa,
 \qquad P_\perp\delta=P_\perp\Delta\log\pi.
 \label{pd:eq:c}
\end{equation}
The common and contrast components are orthogonal, with energies
\begin{equation}
 \begin{aligned}
 Q&=\|P_\perp\delta\|_2^2,\qquad C=Kc^2,\qquad E=\|\delta\|_2^2=Q+C,\\
 C&=K(\Delta\log M)^2+K(\Delta\kappa)^2
       -2K\Delta\log M\,\Delta\kappa.
 \end{aligned}
 \label{pd:eq:energies}
\end{equation}
\end{proposition}
Thus $\delta=c\mathbf1\iff\Delta\pi=0$, with $c=\Delta\log M$ in this case:
a pure common update changes candidate-set mass while preserving the
conditional answer distribution.

\begin{proof}[Proof of Proposition~\ref{pd:prop:readout}]
At each of the two depths $r$, let
$M_r=\sum_{k\in\mathcal C}p_{r,k}$ and $\pi_{r,k}=p_{r,k}/M_r$.
Since $p_{r,k}>0$ and $|\mathcal C|=K$ is fixed,
\begin{equation}
 \begin{aligned}
 s_{r,k}&=\log M_r+\log\pi_{r,k},\\
 \kappa_r
 &=\frac1K\sum_{k\in\mathcal C}\log\frac{1/K}{\pi_{r,k}}
 =-\log K-\frac1K\sum_{k\in\mathcal C}\log\pi_{r,k}.
 \end{aligned}
 \label{pd:eq:mass-proof}
\end{equation}
Consequently,
\[
 \bar s_r=\log M_r-\log K-\kappa_r,\qquad
 c=\bar\delta=\Delta\log M-\Delta\kappa.
\]
Moreover, $P_\perp\mathbf1=0$ and $P_\perp=P_\perp^\top=P_\perp^2$ imply
\[
 \begin{aligned}
 P_\perp\delta&=P_\perp\Delta\log\pi,\\
 \delta&=c\mathbf1+P_\perp\delta,&
 \langle c\mathbf1,P_\perp\delta\rangle&=0.
 \end{aligned}
\]
Hence
\[
 E=\|\delta\|_2^2=Kc^2+\|P_\perp\delta\|_2^2=C+Q,
\]
and
\begin{equation}
 C=K(\Delta\log M)^2+K(\Delta\kappa)^2
       -2K\Delta\log M\,\Delta\kappa.
 \label{pd:eq:mass-concentration-proof}
\end{equation}
\end{proof}
For vocabulary logits $z$, $s_{\mathcal C}=z_{\mathcal C}-\LSE(z)\mathbf1$.
Since $P_\perp\mathbf1=0$ and oscillation is translation-invariant,
\begin{equation}
 P_\perp\Delta s_{\mathcal C}=P_\perp\Delta z_{\mathcal C},\qquad
 \osc(\Delta s_{\mathcal C})=\osc(\Delta z_{\mathcal C}).
 \label{pd:eq:vocabulary-invariance}
\end{equation}
Vocabulary normalization therefore leaves both contrast motion and comparison
changes invariant \citep{gao2018softmax}.

\paragraph{Scope of the probability interpretation.}
For an arbitrary score vector, $a_t=\LSE(s_t)$ and
$\pi_t=\operatorname{softmax}(s_t)$ give the algebraic factorization
$s_t=a_t\mathbf1+\log\pi_t$. It implies the same relation
$\bar\delta=\Delta a-\Delta\kappa$. The identity becomes a statement about
candidate-set probability mass when the selected entries are log probabilities
from one predictive distribution. In the shared-prefix label collection this
condition is built into the forward readout. For the archived MMLU precision
panel it was examined through paired full-vocabulary and prefix-state
comparisons, which retain finite-precision discrepancies detailed in
Section~\ref{pd:app:extended-precision}. In a candidate-conditioned text score, the same algebraic
$a_t$ does not represent the mass of a single vocabulary event.

Mass and concentration both act along the common direction, so their
energy contributions include a cross term. Write
$E_M=K(\Delta\log M)^2$, $E_\kappa=K(\Delta\kappa)^2$ and
$E_\times=-2K\Delta\log M\,\Delta\kappa$. Then
\begin{equation}
 E=Q+E_M+E_\kappa+E_\times,\qquad
 \operatorname{sgn}(E_\times)
 =-\operatorname{sgn}(\Delta\log M\,\Delta\kappa).
 \label{pd:eq:signed-components}
\end{equation}
The interaction distinguishes cancellation from reinforcement within each update.
Figure~\ref{pd:fig:mass} reports mean component energies and their signed
shares $\sum_i X_{i,t}/\sum_i E_{i,t}$,
$X\in\{Q,E_M,E_\kappa,E_\times\}$, at every adjacent Huginn depth.
The four shares sum to one. Figure~\ref{pd:app:fig:ouro-energy} reports the
same decomposition over Ouro's complete native transitions, including the
signed mass--concentration interaction.

\begin{figure}[!htbp]
\centering
\includegraphics[width=\linewidth]{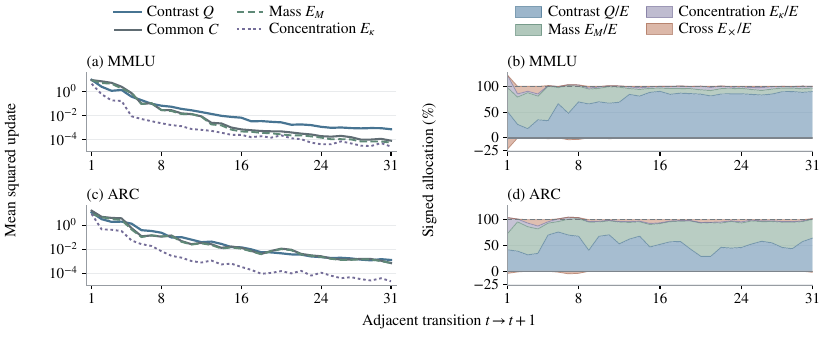}
\caption{Changing energy allocation within contracting updates. All 31 adjacent transitions in Huginn-3.5B (0125), with 128 questions per task under label scoring. Left: mean energies on a logarithmic scale; the signed cross term appears in the right-hand panels. Right: each component's share of total energy, with positive and negative contributions displayed separately. Late MMLU updates allocate a larger share of their energy to contrast even as contrast energy falls; ARC has a different allocation profile. These adjacent increments complement the four-step interval summaries in Table~\ref{pd:tab:energy}.}
\label{pd:fig:mass}
\end{figure}

For the late Ouro MMLU interval $3\to4$,
$(C,E_M,E_\kappa,E_\times)\approx(0.512,0.002,0.514,-0.004)$.
Concentration dominates set-mass energy in this cell. The same common
projection therefore admits different probability compositions across
models and depths.

\begin{figure}[!htbp]
\centering
\begin{subfigure}[t]{.487\linewidth}\centering\includegraphics[width=\linewidth]{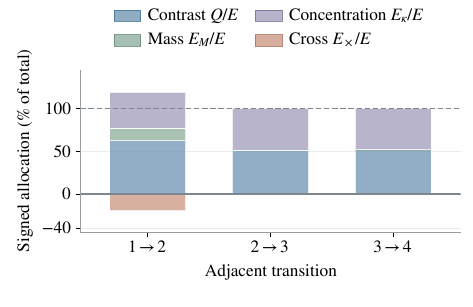}\caption{Ouro-1.4B MMLU.}\label{pd:app:fig:ouro-energy-m}\end{subfigure}\hfill
\begin{subfigure}[t]{.487\linewidth}\centering\includegraphics[width=\linewidth]{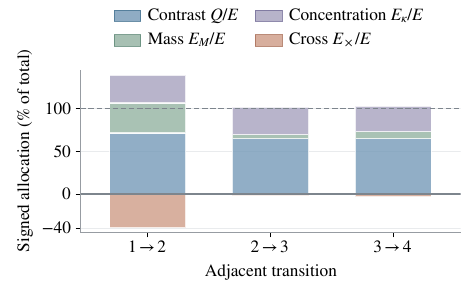}\caption{Ouro-1.4B ARC.}\label{pd:app:fig:ouro-energy-a}\end{subfigure}
\caption{Signed energy allocation over Ouro-1.4B's three native transitions,
using 128 F1 questions per task. The components and denominator match
Figure~\ref{pd:fig:mass}; each bar retains the signed interaction. The
different native horizons describe within-model evolution and do not isolate
architecture as a causal factor.}
\label{pd:app:fig:ouro-energy}
\end{figure}

\subsection{Energy attenuation and weighting}
\label{arxiv:sec:attenuation}
\label{pd:app:depth-weighting}

An increasing contrast share need not indicate growing contrast energy:
the common component may contract faster. To compare these rates,
let $Q_{i,r},C_{i,r}$ be question $i$'s energies averaged over the prescribed
interval starts in window $r\in\{e,\ell\}$, and set $E_{i,r}=Q_{i,r}+C_{i,r}$.
On the same $N$ questions, write $\bar X_r=N^{-1}\sum_i X_{i,r}$ for
$X\in\{Q,C,E\}$. We summarize contrast allocation in two ways:
\[
 D_{{\rm pool},r}=\frac{\bar Q_r}{\bar E_r},\qquad
 q_{i,r}=\frac{Q_{i,r}}{E_{i,r}},\qquad
 D_{{\rm eq},r}=\frac1N\sum_i q_{i,r}.
\]
The pooled share weights questions by their total energy; the equal-question
share averages their individual contrast shares. The pooled quantity requires
$\bar E_r>0$, and the equal-question quantity requires $E_{i,r}>0$ for every
question. If $\bar E_e,\bar E_\ell>0$, then
\begin{equation}
 D_{{\rm pool},\ell}-D_{{\rm pool},e}
 =\frac{\bar Q_\ell\bar C_e-\bar Q_e\bar C_\ell}
 {\bar E_\ell\bar E_e}.
 \label{pd:eq:relative-attenuation}
\end{equation}
For $\bar Q_e,\bar C_e>0$, define
$\rho_Q=\bar Q_\ell/\bar Q_e$ and $\rho_C=\bar C_\ell/\bar C_e$. Consequently,
\begin{equation}
 D_{{\rm pool},\ell}>D_{{\rm pool},e}
 \iff \rho_Q>\rho_C.
 \label{pd:eq:attenuation-ratios}
\end{equation}
In particular, $0\le\rho_C<\rho_Q<1$ gives increasing contrast share
while both absolute energies decrease.

For $E_{i,r}>0$ on every question, the difference between the two summaries
is determined by their covariance under the equal-question empirical measure:
\begin{equation}
 D_{{\rm pool},r}-D_{{\rm eq},r}
 =\frac{\operatorname{Cov}_i(E_{i,r},q_{i,r})}{\bar E_r}.
 \label{pd:eq:weight}
\end{equation}
Questions with larger total energy pull the pooled share toward their own
contrast allocation. For the proof, fix a window $r$ and write
$E_i=E_{i,r}$, $Q_i=Q_{i,r}$ and $d_i=q_{i,r}$. Expectations use the
same equal-question empirical measure.

\begin{proof}
Under the same probability measure on questions, assume $E_i>0$ almost
surely and $0<\Ee E_i<\infty$. Since $0\le d_i\le1$, both $d_i$ and
$E_i d_i=Q_i$ are integrable. By the definition of covariance,
\[
 \begin{aligned}
 \frac{\Ee Q_i}{\Ee E_i}
 &=\frac{\Ee[E_i d_i]}{\Ee E_i}\\
 &=\frac{\Ee E_i\,\Ee d_i+\operatorname{Cov}(E_i,d_i)}{\Ee E_i}
 =\Ee d_i+\frac{\operatorname{Cov}(E_i,d_i)}{\Ee E_i}.
 \end{aligned}
\]
\end{proof}

\paragraph{Early and late interval comparisons.}
In Huginn MMLU, contrast takes a larger share in the late window even as
both contrast and common energies fall. For four-step increments with starts
$4{:}7$ versus $24{:}27$, the 128-question estimates are
$D_{{\rm pool},e}=0.2170$ and $D_{{\rm pool},\ell}=0.7377$, with
$\Delta D_{\rm pool}=0.5207\,[0.4489,0.6010]$.
Table~\ref{pd:tab:energy} gives the corresponding absolute energies and
all four model/task cells. ARC's interval includes zero.

\begin{table}[!htbp]
\centering\small\setlength{\tabcolsep}{4.4pt}
\caption{Absolute energies and relative allocation on 128 questions per cell under label scoring. Mean $Q,C$ have squared-log-probability units. Huginn uses four-step increments with starts $4{:}7$ versus $24{:}27$; Ouro uses $1\to2$ versus $3\to4$. Both $\Delta D$ columns give late-minus-early changes. Mint/bold marks positive $\Delta D$ with lower 95\% limits above zero; complete marginal intervals appear in Table~\ref{pd:app:tab:weighting}.}
\label{pd:tab:energy}
\begin{tabular}{llrrrrrr}
\toprule
 & & \multicolumn{2}{c}{Contrast $Q$} & \multicolumn{2}{c}{Common $C$} & \multicolumn{2}{c}{$\Delta D$ (pp)}\\
\cmidrule(lr){3-4}\cmidrule(lr){5-6}\cmidrule(lr){7-8}
Model & Task & Early & Late & Early & Late & Pooled & Equal\\
\midrule
Huginn-3.5B & ARC & 1.36545 & .00717 & 1.68553 & .00960 & $-1.98$ & $-3.72$\\
 & MMLU & .61025 & .00331 & 2.20178 & .00118 & \cellcolor{palegreen}\textbf{52.07} & \cellcolor{palegreen}\textbf{55.54}\\
Ouro-1.4B & ARC & 4.43761 & .35879 & 1.75132 & .19187 & $-6.55$ & $3.62$\\
 & MMLU & 4.05283 & .55714 & 2.40321 & .51238 & $-10.68$ & \cellcolor{palegreen}\textbf{5.41}\\
\bottomrule
\end{tabular}
\end{table}

The fixed depth prediction concerns both the pooled contrast share
$D=D_{\rm pool}=\Ee Q/\Ee(Q+C)$ and the mass-only approximation error
$\mathcal E_M=\Ee[K(\bar\delta-\Delta\log M)^2]/\Ee C$, with positive
denominators in both ratios. The latter measures the squared error from
approximating common motion by $\Delta\log M$, relative to common energy;
its numerator retains the concentration contribution.
Huginn compares four-step updates starting at depths $4{:}7$ and $24{:}27$;
Ouro compares $1\!\to\!2$ and $3\!\to\!4$.
The two positive MMLU changes occur on the additional Huginn questions but
do not extend to Ouro under its native horizon
(Table~\ref{pd:app:tab:depth-prediction}). Absolute $Q$ and $C$ decrease in
all four model/task conditions. For example, Huginn/MMLU changes from
$(Q,C)=(0.61025,2.20178)$ to $(0.003312,0.001178)$; Ouro/MMLU changes from
$(4.05283,2.40321)$ to $(0.557137,0.512381)$.

\begin{table}[!htbp]
\centering\small
\caption{Original fixed early--late comparisons on $128$ questions per
task and model. Entries are late-minus-early differences with their original
$95\%$ intervals. $D$ is an energy share; $\mathcal E_M$ is a relative
squared error.}
\label{pd:app:tab:depth-prediction}
\begin{tabular}{llrr}
\toprule
Model & Task & $\Delta D$ & $\Delta\mathcal E_M$\\
\midrule
Huginn-3.5B & MMLU & $0.5207\,[0.4489,0.6010]$ & $0.2245\,[0.1094,0.3401]$\\
Huginn-3.5B & ARC & $-0.0198\,[-0.1315,0.0865]$ & $-0.0645\,[-0.1015,-0.0341]$\\
Ouro-1.4B & MMLU & $-0.1068\,[-0.2104,0.0310]$ & $-0.1264\,[-0.2752,0.0150]$\\
Ouro-1.4B & ARC & $-0.0655\,[-0.1449,0.0186]$ & $-0.2945\,[-0.6635,0.0696]$\\
\bottomrule
\end{tabular}
\end{table}

\paragraph{Question weighting.}
In Ouro MMLU, the pooled change is $-0.1068\,[-0.2104,0.0310]$.
The equal-question change is $0.0541\,[0.0054,0.1061]$.
The signs of these changes differ because the summaries assign different
weights to the same questions.

The reported cells have positive aggregated energy
for every question. If this condition fails, the equal-question ratio is
left undefined rather than assigning a value to a zero-energy question.
In Ouro/MMLU the covariance correction changes from $-0.07498$ to
$-0.23592$, reversing the sign relative to equal-question averaging.
Table~\ref{pd:app:tab:weighting} applies both summaries to all four conditions
with a common resampling implementation. Changing the observation interval
also changes effect size: Huginn/MMLU pooled differences are $0.5207$ for
the four-step intervals, $0.4784$ for matched-start one-step intervals, and
$0.1866$ for $8\!\to\!16$ versus $24\!\to\!32$.

\begin{table}[!htbp]
\centering\small
\caption{Pooled and equal-question changes in contrast share, with the
same $5{,}000$ whole-question resamples for all four conditions.}
\label{pd:app:tab:weighting}
\begin{tabular}{llrr}
\toprule
Model & Task & Pooled $\Delta D$ & Equal-question $\Delta D$\\
\midrule
Huginn-3.5B & MMLU & $0.5207\,[0.4488,0.6014]$ & $0.5554\,[0.5137,0.5967]$\\
Huginn-3.5B & ARC & $-0.0198\,[-0.1247,0.0837]$ & $-0.0372\,[-0.0977,0.0219]$\\
Ouro-1.4B & MMLU & $-0.1068\,[-0.2073,0.0315]$ & $0.0541\,[0.0054,0.1061]$\\
Ouro-1.4B & ARC & $-0.0655\,[-0.1470,0.0147]$ & $0.0362\,[-0.0365,0.1057]$\\
\bottomrule
\end{tabular}
\end{table}

\FloatBarrier

\subsection{Marginal means and covariance}
\label{r2:sec:joint-motion}
\label{pd:app:marginal-theory}
\label{pd:app:marginal-results}

The common-energy share also reflects how candidate coordinates vary
together across questions. For a fixed readout, let
$d_{i,t}=s_{i,t+1}-s_{i,t}$ be question $i$'s adjacent score increment.
Group these vectors by task, source subset, choice count $K_g$, and depth;
write $d_{gi}$ for the increment of question $i$ in group $g$, which contains
$n_g\ge1$ questions. Let $P_g=\mathbf1\mathbf1^\top/K_g$ be the common
projector within that group. With common energy
$C_{gi}=(\mathbf1^\top d_{gi})^2/K_g=\|P_gd_{gi}\|_2^2$ and total
energy $E_{gi}=\|d_{gi}\|_2^2$, the observed common-energy fraction is
\begin{equation}
 \Xi_{\rm obs}=\frac{\sum_{g,i}C_{gi}}{\sum_{g,i}E_{gi}}
 =\frac{\sum_{g,i}\|P_gd_{gi}\|_2^2}{E_{\rm tot}},\qquad
 E_{\rm tot}=\sum_{g,i}E_{gi}>0.
 \label{pd:app:eq:xi}
\end{equation}
Write $\mu_g=\Ee_g[d]$ and $\Sigma_g=\operatorname{Cov}_g(d)$, using equal
question weights and empirical divisor $n_g$:
\[
 \mu_g=\frac1{n_g}\sum_i d_{gi},\qquad
 \Sigma_g=\frac1{n_g}\sum_i(d_{gi}-\mu_g)(d_{gi}-\mu_g)^\top.
\]
All covariances are across questions within $g$.

Condition on all increment arrays. For each $g,k$, let $\sigma_{gk}$ be
independent and uniform on the permutations of $\{1,\ldots,n_g\}$, and define
$d^\sigma_{gi,k}=d_{g,\sigma_{gk}(i),k}$.
These permutations preserve every coordinate's empirical distribution and
$E_{\rm tot}$ exactly. Let $\Xi_\sigma$ be the common-energy fraction after
permutation and $\Xi_{\rm marg}=\Ee_\sigma[\Xi_\sigma\mid\{d_{gi}\}]$ its
conditional expectation. All expectations over permutations below are
conditional on these fixed arrays. Then
\begin{equation}
 \begin{aligned}
 \Xi_{\rm obs}
 &=\frac{\sum_g(n_g/K_g)
   [\mathbf1^\top\Sigma_g\mathbf1+(\mathbf1^\top\mu_g)^2]}
 {E_{\rm tot}},\\
 \Xi_{\rm obs}-\Xi_{\rm marg}
 &=\frac1{E_{\rm tot}}\sum_g\frac{n_g}{K_g}
       \sum_{k\ne l}\operatorname{Cov}_g(d_k,d_l).
 \end{aligned}
 \label{pd:app:eq:covariance}
\end{equation}
The reference replaces $\mathbf1^\top\Sigma_g\mathbf1$ by
$\operatorname{tr}(\Sigma_g)$; the sum over $k\ne l$ uses ordered pairs.
If $K_g=K$, $\mu_g=0$, and $\Sigma_g$ is diagonal for every $g$, then
$\Xi_{\rm obs}=\Xi_{\rm marg}=1/K$.
Aligned nonzero means can raise $\Xi_{\rm marg}$ above this value.

\begin{proof}
For every $g,k,l$,
\[
 \sum_{i=1}^{n_g}d_{gi,k}d_{gi,l}
 =n_g\bigl((\Sigma_g)_{kl}+\mu_{gk}\mu_{gl}\bigr).
\]
Independence and uniformity of $\sigma_{gk}$ and $\sigma_{gl}$ for $k\ne l$ give
\[
 \Ee_\sigma\!\left[\sum_{i=1}^{n_g}d^\sigma_{gi,k}d^\sigma_{gi,l}\right]
 =\begin{cases}
 n_g\bigl((\Sigma_g)_{kk}+\mu_{gk}^2\bigr),&k=l,\\
 n_g\mu_{gk}\mu_{gl},&k\ne l.
 \end{cases}
\]
Since $\|P_gd_{gi}\|_2^2=K_g^{-1}\sum_{k,l}d_{gi,k}d_{gi,l}$
and the permutations preserve $E_{\rm tot}$,
\[
 \begin{aligned}
 \Xi_{\rm obs}
 &=\frac1{E_{\rm tot}}\sum_g\frac{n_g}{K_g}
   \left\{\mathbf1^\top\Sigma_g\mathbf1+(\mathbf1^\top\mu_g)^2\right\},\\
 \Ee_\sigma\Xi_\sigma
 &=\frac1{E_{\rm tot}}\sum_g\frac{n_g}{K_g}
   \left\{\operatorname{tr}(\Sigma_g)+(\mathbf1^\top\mu_g)^2\right\}.
 \end{aligned}
\]
Subtracting and using
$\mathbf1^\top\Sigma_g\mathbf1-\operatorname{tr}(\Sigma_g)
 =\sum_{k\ne l}(\Sigma_g)_{kl}$ proves Equation~\ref{pd:app:eq:covariance}.
\end{proof}
The identity permits $n_g=1$ and different $K_g$ across groups; independence
is imposed on the permutations, not on the original increments.

Across all native adjacent ARC increments, Huginn and Recurrent--Llama
have common-energy fractions
77.06\% and 67.51\%, compared with exact marginal expectations 45.37\% and
39.24\%. Their covariance contributions are 31.70 and 28.27 percentage
points. The complete comparisons below also report a finite-permutation
estimate of the reference.

\paragraph{Joint variation across tasks and depths.}
The across-question reference preserves each coordinate's empirical
distribution within task, source subset, choice count, and depth. It answers
a different question from the within-update coordinate permutations in
Section~\ref{arxiv:sec:coordinate-permutations}.

Table~\ref{pd:app:tab:marginal} reports the taskwise native-precision values
from 1,000 such permutations. Large common energy and positive covariance
excess are different properties: Huginn MMLU has both over all increments,
whereas Recurrent-Llama MMLU has a positive common fraction but a negative
excess. Late MMLU, using adjacent destination depths $16{:}32$, has observed
versus reference fractions 14.38/24.83\% in Huginn and 16.05/24.98\% in
Recurrent-Llama. At higher precision, the latter late contrast changes from
$-9.56\,[-12.95,-4.79]$ pp in BF16 to
$-1.23\,[-12.73,9.28]$ pp in full FP32 on the 32-question panel.

\begin{table}[!htbp]
\centering\small
\caption{Native adjacent-increment common energy and its marginal reference
on 1,352 questions per task. Values are percentages. Observed intervals are
95\% whole-question intervals; reference entries are conditional permutation
means. This across-question reference differs from the within-question
displacement-coordinate comparison in Table~\ref{pd:app:tab:permutation}.}
\label{pd:app:tab:marginal}
\begin{tabular}{llrrr}
\toprule
Model & Task & Observed $\Xi$ [95\% interval] & Reference & Difference (pp)\\
\midrule
Huginn-3.5B & ARC & $77.06\,[75.31,78.72]$ & 45.37 & 31.69\\
 & MMLU & $58.01\,[57.18,58.84]$ & 44.32 & 13.69\\
RL-1.4B (R32) & ARC & $67.51\,[64.95,69.82]$ & 39.25 & 28.26\\
 & MMLU & $16.94\,[15.98,17.88]$ & 22.35 & $-5.41$\\
\bottomrule
\end{tabular}
\end{table}

Shared question-level variation can contribute to the covariance sum.
Neither zero excess nor a specified sign establishes coordinate independence.
This reference describes empirical joint structure without assuming
exchangeability under a generative model. The raw-radius-weighted removable fraction $\tau$
(Equation~\ref{pd:eq:tau}) has no corresponding quadratic covariance identity.

\FloatBarrier

\FloatBarrier
\section{Sensitivity to Scoring and Numerical Precision}
\label{arxiv:sec:sensitivity}

The preceding identities characterize a fixed score trajectory; the
observations also depend on how that trajectory is obtained. We examine this
dependence through paired comparisons of scoring construction and numerical
precision, with each scoring format retaining its own endpoint. Uniform
rescaling then identifies which statistics are invariant to score units.

\subsection{Scoring construction and endpoint predictions}
\label{pd:app:format-details}

For Huginn, the paired readouts produce different stability profiles and endpoint
answers: answer-text scoring has lower disagreement with its own endpoint,
while the two readouts often select different endpoint answers. This
comparison changes the prompt and execution as well as the scoring formula.
On the $N=32$ questions per task shared by these conditions, let $y^{(f)}_{i,t}$ be the answer
identity under readout $f\in\{\mathrm{lab},\mathrm{txt}\}$, where label
scoring displays the options and answer-text scoring omits the options block.
With $T=32$ and $\mathcal T=\{8,16,24\}$, define
\begin{equation}
 \begin{aligned}
 \varepsilon_f&=\frac1{N|\mathcal T|}\sum_{i=1}^N\sum_{t\in\mathcal T}
   \mathbf1\{y^{(f)}_{i,t}\ne y^{(f)}_{i,T}\},\qquad
 A_f=\frac1N\sum_i\mathbf1\{y^{(f)}_{i,T}=y_i^\star\},\\
 a_{\rm lab,txt}&=\frac1N\sum_i
   \mathbf1\{y^{(\mathrm{lab})}_{i,T}=y^{(\mathrm{txt})}_{i,T}\}.
 \end{aligned}
 \label{pd:eq:format-summaries}
\end{equation}
Here $\varepsilon_f$ is the fraction of sampled intermediate answers that
disagree with their own endpoint, $A_f$ is endpoint accuracy against the
ground-truth answer $y_i^\star$, and $a_{\rm lab,txt}$ is endpoint agreement
between the two readouts. The paired text-minus-label differences in
$\varepsilon_f$ are
$-28.13$ pp $[-39.58,-17.71]$ on ARC and
$-29.17$ pp $[-41.67,-16.67]$ on MMLU;
$a_{\rm lab,txt}=10/32$ and $12/32$, respectively.
Table~\ref{pd:tab:format} also reports the normalized earliest-depth areas
$W_j$ for raw, mean-centered, and quotient tests.

\begin{table}[!htbp]
\centering\small\setlength{\tabcolsep}{4pt}
\caption{Same-question prediction and geometry: Huginn-3.5B (0125), $N=32$/task, $T=32$, candidate depths $\{8,16,24\}$. F1 denotes label scoring and F2 answer-text scoring without displayed options (Equation~\ref{pd:eq:f2}). Disagreement averages $\mathbf1\{y_t\ne y_{32}\}$; the last columns give $W$ (\%). Blue marks paired disagreement reductions versus label scoring; mint marks quotient gains over mean centering, each with 95\% intervals excluding zero.}
\label{pd:tab:format}
\begin{tabular}{llcrrrr}
\toprule
Task & Format & Correct / $N$ & Disagreement & Raw & Mean & Quotient\\
\midrule
ARC & F1 & $8/32$ & 34.38 & 18.75 & 29.69 & 32.81\\
 & F2 & $14/32$ & \cellcolor{paleblue}\textbf{6.25} & 42.97 & 53.12 & \cellcolor{palegreen}\textbf{57.03}\\
MMLU & F1 & $10/32$ & 36.46 & 19.53 & 24.22 & 25.78\\
 & F2 & $15/32$ & \cellcolor{paleblue}\textbf{7.29} & 34.38 & 45.31 & \cellcolor{palegreen}\textbf{50.78}\\
\bottomrule
\end{tabular}
\end{table}

Under answer-text scoring, $W_{\rm q}-W_{\rm raw}$ is 14.06 pp on ARC and
16.41 pp on MMLU;
$W_{\rm q}-W_{\rm mean}$ is $3.91\,[0.78,8.59]$ and
$5.47\,[0.78,11.72]$ pp, respectively. Both quotient-over-mean intervals
exclude zero, whereas both paired accuracy intervals for
$A_{\rm txt}-A_{\rm lab}$ include zero. Lower endpoint disagreement therefore
does not by itself imply an accuracy gain. The sum and token-normalized scores below complete the comparison, while
Figure~\ref{pd:fig:prediction} shows the prediction profiles.

\begin{figure}[!htbp]
\centering
\begin{subfigure}[t]{.487\linewidth}\centering\includegraphics[width=\linewidth]{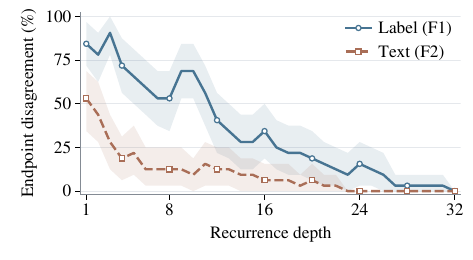}\caption{ARC-Challenge}\label{pd:fig:prediction-arc}\end{subfigure}\hfill
\begin{subfigure}[t]{.487\linewidth}\centering\includegraphics[width=\linewidth]{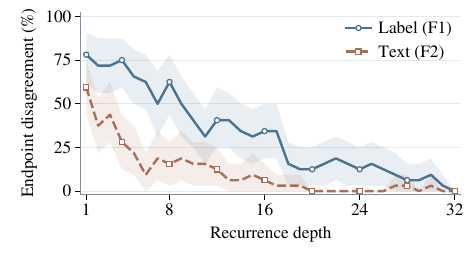}\caption{MMLU}\label{pd:fig:prediction-mmlu}\end{subfigure}
\caption{Distinct prediction trajectories under label and answer-text scoring. Huginn-3.5B (0125) on the same 32 questions per task, with label scoring (F1) and character-normalized answer-text scoring (F2). Each curve compares depth $t$ with its own depth-32 winner; shading shows pointwise 95\% whole-question bootstrap intervals. The scoring formulas are defined in Section~\ref{arxiv:sec:scoring}; Table~\ref{pd:tab:format} gives the paired comparison.}
\label{pd:fig:prediction}
\end{figure}

\paragraph{Score normalization.}

Using the character-normalized F2 score in Equation~\ref{pd:eq:f2},
the alternative normalizations are
\begin{equation}
 s^{(L)}_{t,k}=\frac{\ell_k}{L_k}s^{\mathrm{txt}}_{t,k},
 \qquad L_k\in\{1,n_k,\ell_k\},
 \label{pd:app:eq:text-scoring}
\end{equation}
where $L_k=1,n_k,\ell_k$ gives the sum, token mean, and character mean,
respectively, on the same answer continuations. Character mean is
the main F1/F2 comparison; Table~\ref{pd:app:tab:text-variants} retains all
three variants. Actual target-token log probabilities supply the sums:
later-position columns for the first candidate token do not supply those
targets. Reconstructed scores differ from stored reductions by at most
$1.01\times10^{-5}$, with identical selected answers.

\begin{table}[!htbp]
\centering\footnotesize
\caption{All population scoring variants on the same $32$ questions per
task in Huginn-3.5B (0125). $\Xi$ and $\tau$ are the common-energy and
removable-radius shares (Equations~\ref{pd:app:eq:xi} and~\ref{pd:eq:tau}),
computed as adjacent-increment ratios of sums, in percent.
$\Delta W=W_{\rm q}-W_{\rm raw}$ is in percentage points, relative to each
format's own $T=32$ endpoint; brackets give $95\%$ paired intervals.}
\label{pd:app:tab:text-variants}
\setlength{\tabcolsep}{4pt}
\begin{tabular}{llrrrrr}
\toprule
Task & Score & Correct & $\Xi$ & $\tau$ & $\Delta W$: $8,16,24$ & $\Delta W$: $1{:}31$\\
\midrule
ARC & F1 & $8$ & $59.65$ & $41.19$ & $14.06\,[7.81,21.09]$ & $16.41\,[10.84,23.05]$\\
 & F2 char & $14$ & $71.05$ & $52.42$ & $14.06\,[9.38,19.53]$ & $21.58\,[15.82,27.64]$\\
 & F2 sum & $9$ & $72.04$ & $52.14$ & $17.19\,[10.16,25.00]$ & $21.19\,[15.33,27.34]$\\
 & F2 token & $11$ & $70.92$ & $52.08$ & $11.72\,[6.25,17.19]$ & $18.07\,[12.89,23.34]$\\
\midrule
MMLU & F1 & $10$ & $60.65$ & $44.09$ & $6.25\,[2.34,10.94]$ & $20.12\,[12.01,29.00]$\\
 & F2 char & $15$ & $80.43$ & $53.47$ & $16.41\,[9.38,24.22]$ & $23.14\,[15.92,30.86]$\\
 & F2 sum & $16$ & $91.73$ & $56.57$ & $21.09\,[14.06,29.69]$ & $22.75\,[15.23,31.35]$\\
 & F2 token & $13$ & $77.95$ & $53.15$ & $17.97\,[10.94,25.00]$ & $18.85\,[11.33,26.96]$\\
\bottomrule
\end{tabular}
\end{table}

\paragraph{Endpoint-conditioned comparisons.}
The paired text-minus-label accuracy differences are
$18.75\,[-3.12,40.62]$ and $15.62\,[-3.12,34.38]$ pp. Conditioning on
each format's own correct endpoint gives quotient-over-raw depth increments
of $9.38\,[3.12,18.75]$ pp for ARC/F1 ($n=8$),
$17.86\,[10.71,25.00]$ for ARC/F2 ($n=14$),
$7.50\,[0.00,15.00]$ for MMLU/F1 ($n=10$), and
$15.00\,[8.33,23.33]$ for MMLU/F2 ($n=15$).
Only four ARC and seven MMLU questions belong to both correct-endpoint
subsets. On the paired eight-question precision panels, F1/F2-char depth
increments change from $12.50/9.38$ to $18.75/18.75$ pp on ARC and from
$12.50/18.75$ to $12.50/21.88$ pp on MMLU between BF16 and FP32 recurrence.
All depth quantities normalize recurrence within a scoring format; F2
executes multiple answer continuations whereas F1 uses one shared prefix.

\subsection{Numerical precision and token composition}
\label{pd:app:extended-precision}
The archived-format precision study evaluates the same $32$ questions per task
in both Raven-lineage checkpoints. The three arms use native BF16, an FP32
head applied to the same BF16 hidden values, and full-FP32 recurrence with
the same BF16-rounded weights and saved initial tensors. These arithmetic
changes can affect translation identities in finite precision
\citep{blanchard2021logsumexp}.

We use the retained-radius ratio $\rho_{i,t}$ and removable fraction
$\tau$ defined in Equation~\ref{pd:eq:tau}.
Thus $\tau$ is a raw-radius-weighted removable fraction, whereas the
atom at zero in Figure~\ref{pd:fig:precision}(a,b) counts exact nonzero translations.

\begin{figure}[!htbp]
\centering
\includegraphics[width=\linewidth]{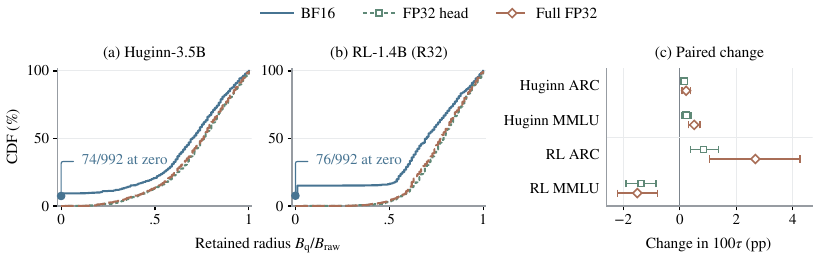}
\caption{Exact plateaus disappear while aggregate translation persists. (a,b) MMLU distributions of $B_{\rm q}/B_{\rm raw}$ over nonzero adjacent increments, sampled on a 0.01 grid; an atom at zero denotes exact translation. (c) Paired changes in raw-radius-weighted $100\tau$, with 95\% whole-question intervals. Each checkpoint has the same 32 questions per task across arms. FP32-head preserves BF16 hidden states; full FP32 changes upstream arithmetic with the same rounded weights.}
\label{pd:fig:precision}
\end{figure}

Table~\ref{pd:app:tab:precision} gives the aggregate removable-radius shares
and $95\%$ intervals from $2{,}000$ paired whole-question resamples.
Native MMLU contains $74/992$ exact zero-span, nonzero-raw increments in Huginn
and $76/992$ in Recurrent--Llama. Both counts become zero in each
higher-precision arm, while aggregate translation removal remains substantial
with model-dependent changes. Thus exact candidate-contrast plateaus depend
on arithmetic precision; these comparisons do not establish precision
invariance of every full-text depth area or mass--concentration result.
Figure~\ref{pd:app:fig:arc-precision} completes the retained-radius distributions
for ARC under the same plotting definitions.

\begin{table}[!htbp]
\centering\small
\caption{Adjacent-increment radius reduction under paired precision changes
in the archived scoring formats. Values are percentages with $95\%$ intervals.
H: Huginn-3.5B (0125); RL: Recurrent-Llama-1.4B (R32).}
\label{pd:app:tab:precision}
\setlength{\tabcolsep}{4pt}
\begin{tabular}{llrrr}
\toprule
Model & Task & BF16 & FP32 head & Full FP32\\
\midrule
H & ARC & $51.88\,[49.13,54.57]$ & $52.04\,[49.23,54.76]$ & $52.11\,[49.29,54.89]$\\
H & MMLU & $44.34\,[42.57,46.02]$ & $44.57\,[42.75,46.33]$ & $44.84\,[42.98,46.63]$\\
RL & ARC & $50.08\,[45.67,54.90]$ & $50.92\,[46.23,56.07]$ & $52.76\,[47.14,58.72]$\\
RL & MMLU & $24.66\,[23.12,26.11]$ & $23.28\,[21.82,24.71]$ & $23.15\,[21.42,24.86]$\\
\bottomrule
\end{tabular}
\end{table}

For ARC, split each score into its first-token and remaining-token terms,
and denote their cross-choice mean increments by $u$ and $v$. The identity
$K(u+v)^2=Ku^2+Kv^2+2Kuv$ yields first/later/cross contributions of
$46.38/35.63/17.99\%$ for Huginn and $44.26/27.77/27.97\%$ for
Recurrent--Llama at native precision. The signed cross term retains the
interaction between token contributions. Matched prefix-only executions also
show numerical shape sensitivity: the maximum vocabulary-log-probability
discrepancy falls from $0.372150$ to $5.8\times10^{-5}$ for Huginn, and from
$0.254886$ to $3.7\times10^{-5}$ for Recurrent--Llama, between BF16 and full
FP32. The pre-head hidden values change as well.

\begin{figure}[!htbp]
\centering
\begin{subfigure}[t]{.487\linewidth}\centering\includegraphics[width=\linewidth]{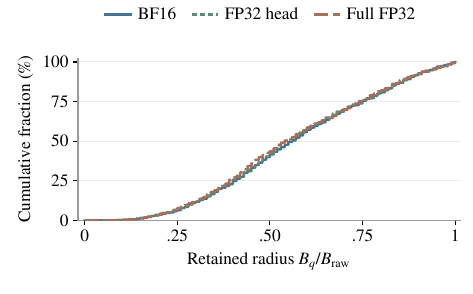}\caption{Huginn-3.5B ARC.}\label{pd:app:fig:arc-precision-h}\end{subfigure}\hfill
\begin{subfigure}[t]{.487\linewidth}\centering\includegraphics[width=\linewidth]{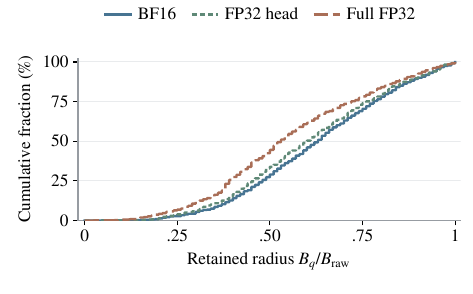}\caption{RL-1.4B (R32) ARC.}\label{pd:app:fig:arc-precision-r}\end{subfigure}
\caption{Precision-dependent retained-radius distributions on ARC. Each
checkpoint uses the same 32 questions across all arithmetic arms; ratios
are sampled on the same 0.01 grid as Figure~\ref{pd:fig:precision}(a,b).
The distribution weights each nonzero adjacent increment equally, whereas
$\tau$ weights removable fractions by the corresponding raw radius.}
\label{pd:app:fig:arc-precision}
\end{figure}

\FloatBarrier
\subsection{Score scale and centering}
\label{pd:app:statistics}
\label{pd:app:scale}
Uniform positive rescaling tests which properties depend on score units.
For $\alpha>0$, every selected answer is preserved. Raw, mean-centered and
quotient radii, directed bounds, reserves, and margins scale by $\alpha$.
On a fixed candidate-depth set,
\begin{equation}
 \begin{aligned}
 (Q,C)(\alpha\delta)&=\alpha^2(Q,C)(\delta),\\
 D_j(\alpha s)&=D_j(s),\qquad
 j\in\{\mathrm{raw,mean,q,dir,res}\}.
 \end{aligned}
 \label{r9:eq:scale-equivariance}
\end{equation}
For $Q+C>0$, the contrast share $Q/(Q+C)$ is therefore invariant.
The same holds for $\tau$ when its raw-radius sum is positive and $\Xi$
when its total energy is positive. LSE centering is not homogeneous:
in general, $\LSE(\alpha s_T)-\LSE(\alpha s_t)\ne
\alpha(\LSE(s_T)-\LSE(s_t))$.

The LSE representative in the original score units is
$a_\alpha(s)=\alpha^{-1}\LSE(\alpha s)$ \citep{gao2018softmax}.
For a fixed finite score vector and the uniform distribution $U_K$ on its
$K$ coordinates, as $\alpha\to0^+$,
\begin{equation}
 a_\alpha(s)-\frac{\log K}{\alpha}
 =\bar s+\frac\alpha2\operatorname{Var}_{U_K}(s)+O(\alpha^2).
 \label{r2:eq:temperature}
\end{equation}
At fixed $K$ across depths, the divergent $\log K/\alpha$ terms cancel,
so $\Delta a_\alpha(s)\to\Delta\bar s$. The following derivation gives an
explicit remainder bound for the expansion.

\begin{proof}[Proof of Equation~\ref{r2:eq:temperature}]
Fix $s\in\R^K$ and write
\[
 f(u)=\log\!\left(\frac1K\sum_{k=1}^K e^{u s_k}\right),\qquad
 \pi_{u,k}=\frac{e^{u s_k}}{\sum_l e^{u s_l}},\qquad L=\osc(s).
\]
The function $f$ is smooth on $\R$, with
\[
 \begin{aligned}
 f(0)&=0,& f'(u)&=\Ee_{\pi_u}s,\\
 f''(u)&=\operatorname{Var}_{\pi_u}(s),&
 f^{(3)}(u)&=\Ee_{\pi_u}(s-\Ee_{\pi_u}s)^3.
 \end{aligned}
\]
Since $\Ee_{\pi_u}s\in[\min_k s_k,\max_k s_k]$,
\[
 |f^{(3)}(u)|\le\Ee_{\pi_u}|s-\Ee_{\pi_u}s|^3\le L^3.
\]
For $\alpha>0$, Taylor's formula with integral remainder yields
\[
 f(\alpha)=\alpha\bar s
   +\frac{\alpha^2}{2}\operatorname{Var}_{U_K}(s)
   +\frac12\int_0^\alpha(\alpha-u)^2f^{(3)}(u)\,du.
\]
Using $a_\alpha(s)=(\log K+f(\alpha))/\alpha$,
\begin{equation}
 \begin{aligned}
 \left|a_\alpha(s)-\frac{\log K}{\alpha}
      -\bar s-\frac\alpha2\operatorname{Var}_{U_K}(s)\right|
 &\le\frac{L^3}{2\alpha}\int_0^\alpha(\alpha-u)^2\,du\\
 &=\frac{\alpha^2L^3}{6}.
 \end{aligned}
 \label{pd:app:eq:temperature-remainder}
\end{equation}
\end{proof}
For fixed $K$, the $\log K/\alpha$ terms cancel in a two-depth difference,
and the two remainder bounds add. This gives the $O(\alpha^2)$ expansion
as $\alpha\to0^+$, uniformly over score vectors with bounded span.

Across the native population, scales
$\alpha\in\{0.1,0.25,0.5,1,2,4,10\}$ preserve all winners and the
homogeneous statistics. With $b_\alpha=\Delta a_\alpha(s)$, the aggregate
excess over the quotient radius after LSE centering, normalized by the original
raw radii, is $\sum 2|\operatorname{mid}(d)-b_\alpha|/\sum2\|d\|_\infty$
for a positive denominator (Equation~\ref{pd:eq:centering}). It changes
from 14.99\% at $\alpha=0.1$ to 21.87\% at $\alpha=1$ and 50.89\% at
$\alpha=10$ for Recurrent-Llama MMLU. Its mean-centered counterpart remains
15.08\%. Mean centering has less residual than LSE centering in all four
model/task cells at native scale. Thus an LSE residual depends on the chosen
score scale; its set-mass interpretation additionally requires a shared
predictive distribution (Section~\ref{pd:app:mass-interpretation}).

\FloatBarrier

\FloatBarrier
\section{From Descriptive Geometry to Prediction}
\label{arxiv:sec:prediction}

The preceding analyses use completed trajectories to characterize score
changes and answer preservation. We now examine what the readout geometry
predicts on held-out questions: first the second moments and common-energy
allocation of updates, then endpoint agreement from an observed prefix.
The fitted comparisons retain their training and evaluation split; the
prefix study uses its own previously specified policy evaluation.

\subsection{Predicting readout moments with matched energy}
\label{pd:app:readout-prediction}
The moment and allocation studies use the previously inspected Huginn
paired-format population of $32$ questions per task
(Appendix~\ref{pd:app:populations}). Question-level folds separate fitting
from evaluation within this population. We compare an independent-logit
innovation reference (N1) with a reference
that applies the model's actual output head to isotropic hidden innovations
(N2). Linear propagation and direct nonlinear readout evaluation use the
same observation map. Four question-level folds contain $24$ training and
eight held-out questions per task; all depths and representations of a
question share its fold. Nonlinear evaluation uses $512$ bounded Rademacher
draws at the five fixed transition anchors $4,8,16,24,31$.

For the following moment and allocation comparisons, $y\in\R^K$ denotes
an observed score increment.
For predicted mean $\mu_j$ and positive-semidefinite covariance $C_j$ with
$\Ee_{\rm train}\operatorname{tr}(C_j)>0$, training-energy matching sets
\begin{equation}
 \alpha_j=
 \frac{\Ee_{\rm train}\|y-\mu_j\|^2}
      {\Ee_{\rm train}\operatorname{tr}(C_j)},
 \qquad \widetilde C_j=\alpha_jC_j.
 \label{pd:app:eq:energy-match}
\end{equation}
The means remain fixed. Losses compare predicted raw second moments with
the observed increment outer products, using the training normalization and
averaging depths within questions. Table~\ref{pd:app:tab:moment-prediction}
reports the comparisons after matching training energy. Linear N2 has higher
loss than N1 on ARC/F1; their MMLU/F1 difference includes zero.
Direct nonlinear evaluation improves ARC/F1 relative to linear N2;
its comparison with matched N1 has an interval extending to zero. F2 here
denotes unnormalized text-score sums.

\begin{table}[!htbp]
\centering\footnotesize
\caption{Paired raw-second-moment loss differences after training-energy
matching. Negative values favor the first-named method. Brackets are
$95\%$ intervals conditional on the fitted models.}
\label{pd:app:tab:moment-prediction}
\setlength{\tabcolsep}{3pt}
\begin{tabular}{@{}llrrr@{}}
\toprule
Task & Score & Linear N2 $-$ N1 & Nonlinear N2 $-$ N1 & Nonlinear $-$ linear N2\\
\midrule
ARC & F1 & $0.1921\,[0.0925,0.2991]$ & $-1.2011\,[-2.3992,0.0046]$ & $-1.3932\,[-2.6346,-0.1480]$\\
MMLU & F1 & $0.0664\,[-0.0258,0.1820]$ & $0.0882\,[-1.1059,1.3802]$ & $0.0218\,[-1.2305,1.3868]$\\
ARC & F2 sum & $-0.0513\,[-0.1099,0.0106]$ & $-0.0873\,[-0.4275,0.1502]$ & $-0.0361\,[-0.4178,0.2301]$\\
MMLU & F2 sum & $0.0706\,[-0.0141,0.1822]$ & $-2.3638\,[-7.2797,0.2055]$ & $-2.4344\,[-7.4285,0.1509]$\\
\bottomrule
\end{tabular}
\end{table}

\FloatBarrier
\subsection{Predicting allocation after matching energy}
\label{arxiv:sec:allocation}
To separate allocation from total scale, set $e=y-\mu$ using the same
training empirical mean for every method, $u=\mathbf1/\sqrt K$,
$A=(u^\top e)^2$, and $B=\|e\|^2-A\ge0$.

\begin{samepage}
Let $p\in[0,1]$ predict the common share of residual energy.
With training residual energy $E_0>0$, the weighted loss
\begin{equation}
 \ell(p;A,B)=\frac{A+B}{E_0}
       \left(p-\frac{A}{A+B}\right)^2
 \label{pd:app:eq:allocation-loss}
\end{equation}
is defined to be zero when $A+B=0$.
\end{samepage}
\begin{proof}[Proof of optimal allocation under the weighted loss]
Condition on the fitted training quantities, including $E_0>0$.
For a fixed $x$, write
\[
 S=A+B,\qquad s_x=\Ee[S\mid x],\qquad a_x=\Ee[A\mid x].
\]
Assume $0<s_x<\infty$, and define $A^2/S=0$ on $\{S=0\}$.
Since $0\le A\le S$ and $0\le A^2/S\le S$,
\[
 0\le a_x\le s_x,\qquad c_x:=\Ee[A^2/S\mid x]<\infty.
\]
For $p\in\R$, expansion of Equation~\ref{pd:app:eq:allocation-loss} gives
\[
 R_x(p):=\Ee[\ell(p;A,B)\mid x]
       =\frac{s_xp^2-2a_xp+c_x}{E_0}.
\]
Setting $p^\star=a_x/s_x\in[0,1]$ yields
\[
 R_x(p)-R_x(p^\star)=\frac{s_x}{E_0}(p-p^\star)^2.
\]
As $s_x/E_0>0$, $p^\star$ is the unique minimizer on both $\R$ and $[0,1]$.
If $s_x=0$, then $S=0$ conditionally almost surely, so $R_x(p)=0$ for every $p$.
\end{proof}
The empirical comparator uses the training
ratio of sums. A two-parameter ridge predictor uses the standardized linear-N2
allocation as its sole feature, with unit slope penalty and unpenalized
intercept. Table~\ref{pd:app:tab:allocation} reports the direct nonlinear
prediction and ridge against the same constant. The earlier two-block loss uses a different objective,
\[
 \ell_2(p;A,B)=(p-A/E_0)^2+(1-p-B/E_0)^2.
\]
\begin{proof}[Proof of the two-block minimizer]
Assume $\Ee[A^2+B^2\mid x]<\infty$ and define
\[
 p_0=\frac12+\frac{\Ee[A-B\mid x]}{2E_0}.
\]
Direct expansion gives, for every $p\in\R$,
\[
 \Ee[\ell_2(p;A,B)-\ell_2(p_0;A,B)\mid x]=2(p-p_0)^2.
\]
Consequently,
\[
 \arg\min_{p\in\R}\Ee[\ell_2(p;A,B)\mid x]=\{p_0\},\qquad
 \arg\min_{p\in[0,1]}\Ee[\ell_2(p;A,B)\mid x]=\{\Pi_{[0,1]}p_0\}.
\]
\end{proof}
All fitted predictions lie in $[0,1]$. The earlier loss gave an MMLU/F2
ridge difference of $-0.1029\,\allowbreak[-0.2699,-0.0078]$.
Under Equation~\ref{pd:app:eq:allocation-loss}, the corresponding
weighted-ridge interval extends to zero. All predictive intervals use
$2{,}000$ paired question resamples with fits and simulated draws held fixed.

\begin{table}[!htbp]
\centering\small
\caption{Energy-allocation loss relative to the training constant under
Equation~\ref{pd:app:eq:allocation-loss}. Negative differences favor the
candidate; the four main task/format conditions are shown.}
\label{pd:app:tab:allocation}
\begin{tabular}{llrr}
\toprule
Task & Score & Nonlinear geometry $-$ constant & Weighted ridge $-$ constant\\
\midrule
ARC & F1 & $0.05625\,[0.02391,0.08962]$ & $-0.01231\,[-0.03936,0.00810]$\\
ARC & F2 sum & $0.09697\,[0.04030,0.15908]$ & $0.00369\,[-0.00844,0.01666]$\\
MMLU & F1 & $0.06529\,[0.02726,0.10676]$ & $-0.00231\,[-0.01554,0.00991]$\\
MMLU & F2 sum & $0.13603\,[0.01379,0.28691]$ & $-0.03744\,[-0.10122,0.00030]$\\
\bottomrule
\end{tabular}
\end{table}

\FloatBarrier
\subsection{Predicting endpoint agreement from a prefix}
\label{pd:app:prefix-prediction}
Completed-trajectory geometry and prefix prediction can also be compared
empirically. An earlier nested, trajectory-held-out prefix study on the native
Huginn collection uses $T=32$, a four-step observation window and candidate
depths $4{:}31$. At the 2.5\% target disagreement level, normalized work
reductions are 21.72\% for raw, 23.93\% for mean-centered and 22.97\% for
quotient ranking, with empirical disagreement among exits of 2.46, 2.54
and 2.54\%, respectively. The paired quotient difference is
$1.26\,[1.08,1.43]$ pp against raw and
$-0.95\,[-1.10,-0.82]$ pp against mean; a train-selected richer comparator
is ahead by $1.36\,[1.01,1.71]$ pp. These earlier policy results concern
observed-prefix rankings rather than the realized increments used throughout
the geometric analysis.

\FloatBarrier

\FloatBarrier
\section{Discussion}

The two decompositions resolve different aspects of a prediction trajectory. Common and contrast components describe what the score update contains; the shared-distribution factorization gives common motion a probability interpretation through mass and concentration. Winner-directed motion and competitor gaps determine whether the update crosses a decision boundary. Consequently, neither update magnitude nor an energy share alone specifies its decision consequence. The measured direction and pairing contributions under label and answer-text scoring show that this distinction accounts for a substantial part of the gap between magnitude bounds and realized answer preservation.

The statistical references further distinguish mathematical structure from
predictive performance. In the matched-label comparison, actual Huginn updates
preserve the winner less often than the exact displacement-coordinate
permutation reference. In the native-precision archive, late MMLU common
energy falls below its coordinate-marginal expectation in both checkpoints. After energy
matching, direct readout geometry does not outperform the constant allocation
reference in the four reported conditions. These results constrain what can
be inferred from a large common component or a favorable retrospective
margin: neither alone identifies a predictive model of future updates.
In the prefix study, quotient ranking improves on the raw comparator but
remains behind mean centering and the richer comparator selected on training
data. Retrospective geometric gains therefore need a separate predictive
evaluation before they can inform a stopping policy.

The study covers multiple-choice predictions from three checkpoints in two model families. Completed-trajectory depth areas describe geometric headroom; an online policy would additionally need to predict future relative updates from a prefix. Interventions on recurrent states could identify the neural processes that generate the mass and concentration terms. Extending the analysis to open-ended generation requires accounting for changing candidate sets and the interaction between recurrent depth and generated context.

\section*{Reproducibility Statement}
Section~\ref{arxiv:sec:scoring} defines the scoring formulas.
Appendix~\ref{pd:sec:observation} specifies fixed model revisions, endpoints,
precision settings, resampling units, and question selection and overlap. Mathematical identities are proved beside their
statements. Each empirical analysis reports its population, comparisons and
uncertainty intervals; predictive analyses in Section~\ref{arxiv:sec:prediction}
retain their fitted training quantities and held-out evaluation roles.

\clearpage
\appendix
\setcounter{table}{0}
\renewcommand{\thetable}{\Alph{section}.\arabic{table}}
\renewcommand{\theHtable}{\Alph{section}.\arabic{table}}
\section{Experimental Design}
\label{pd:sec:observation}

The preceding analyses use complete trajectories under fixed readout and
endpoint choices. We collect here the model configurations, question
populations and resampling procedures needed to reproduce those comparisons.
Score constructions are defined in Section~\ref{arxiv:sec:scoring}; each result
identifies its applicable population. The distinctions between previously
inspected questions and the separate new-question population remain essential
to interpreting the empirical evidence.

\subsection{Checkpoints and trajectory populations}
\label{pd:app:populations}

We evaluate Huginn-3.5B, Recurrent-Llama-1.4B and Ouro-1.4B on ARC-Challenge
\citep{clark2018arc} and MMLU \citep{hendrycks2021mmlu}, with fixed evaluation
rules \citep{biderman2024evaluation}. Huginn and Recurrent-Llama have distinct
weights and training runs within the Raven lineage; Ouro supplies a second
model family. Parameter counts refer to shared weights, separately from
recurrence depth: Huginn's official 3.5B size counts unique parameters;
Recurrent-Llama has 1,385,228,288 parameters after its Llama-3.2-1B retrofit
\citep{mcleish2025recurrent}, including untied embeddings, with R32 denoting its
training recurrence. Ouro's 1.4B shared parameters are independent of its
four-loop inference horizon. Tables~\ref{pd:tab:data} and
\ref{pd:app:tab:checkpoints} give the collections and pinned revisions.

\begin{table}[htb]
\centering\small\setlength{\tabcolsep}{4pt}
\caption{Trajectory collections; $N$ is the number of questions per task. H: Huginn-3.5B; RL: Recurrent-Llama-1.4B; O: Ouro-1.4B. Label, text without options and text with options correspond to F1, F2 and F3. H/O questions are paired for label and displayed-text scoring; H label/text comparisons use the same questions.}
\label{pd:tab:data}
\begin{tabular}{llrcl}
\toprule
Collection & Models & $N$ & $T$ & Scoring condition\\
\midrule
Native trajectories & H, RL & 1,352 & 32 & Native evaluator\\
Matched labels & H, O & 128 & 32, 4 & Label; options displayed\\
Paired label/text & H & 32 & 32 & Label / text without options\\
Text, four arrangements & H, O & 32 & 32, 4 & Text; options displayed\\
Precision comparison & H, RL & 32 & 32 & Native evaluator\\
\bottomrule
\end{tabular}
\end{table}

\begin{table}[!htbp]
\centering\small\setlength{\tabcolsep}{5pt}
\caption{Pinned model configuration. Parameter counts refer to shared weights,
not their repeated use across depths. The Huginn/Ouro matched-label and
answer-text collections use BF16 recurrence with an FP32 output head; the
Huginn/Recurrent-Llama archive uses native BF16. Paired arithmetic arms are
specified in Section~\ref{pd:app:extended-precision}.}
\label{pd:app:tab:checkpoints}
\begin{tabular}{llrc}
\toprule
Model / checkpoint & Pinned revision & Parameters & Evaluated $T$\\
\midrule
Huginn-3.5B (0125) & \texttt{bb6621b65e90} & 3.5B & 32\\
Recurrent-Llama-1.4B (R32) & \texttt{a5f6f126e9d9} & 1.4B & 32\\
Ouro-1.4B & \texttt{574fa66cb8bf} & 1.4B & 4\\
\bottomrule
\end{tabular}
\end{table}

\paragraph{Archived and paired-format populations.}
The archive contains $2{,}704$ questions, equally divided between ARC-Challenge
and MMLU. Huginn and Recurrent-Llama each contribute a complete $32$-depth
trajectory for every question: $86{,}528$ readout vectors and $83{,}824$ adjacent
increments per checkpoint. Archived ARC scores sum continuation log probabilities
and divide by character length; each MMLU continuation contains one answer token.
Four ARC questions have three choices and two have five; all remaining questions
have four. Coordinate permutations are stratified by choice count, task and depth.
A fixed hash selects $32$ questions per task from this archive for the paired-format
study. Huginn evaluates each in F1 and F2 over all $32$ depths; its paired full-FP32
subset contains eight questions per task. These previously inspected data support
the readout-reference and same-question format comparisons, rather than a
new-question confirmation.

\paragraph{Separate matched-label depth population.}
The depth study uses $128$ additional questions per task, absent from the project's
preceding question manifests. Its early--late prediction was specified before
collecting the Huginn outputs. Ouro evaluates the same questions with identical
F1 prompt text at its native four-loop horizon; tokenization remains model-specific.
A fixed $16$-question subset per task and model supplies the paired precision analysis.

\paragraph{Full answer-text geometry on reused questions.}
This collection reuses the same $32$ previously inspected questions per task as
the paired-format study. Huginn and Ouro score every complete answer with character
normalization and all options displayed (F3). Four fixed arrangements,
$(1,2,3,4)$, $(2,4,1,3)$, $(4,3,2,1)$ and $(3,1,4,2)$, balance each answer across
positions while retaining each continuation's semantic identity. Every question
contributes, including those with wrong endpoints. The primary analysis applies
an FP32 output head to each complete sequence after BF16 recurrence, uses native
endpoints $T=32/4$ and the quarter-depth grid, and takes the output at each
recurrent depth without Ouro's gate mixture. The readout formula, arrangements,
depth grid and comparisons were fixed before the Ouro collection. This extends
the analysis to another model family on reused questions without changing their
data role.

Unless a precision comparison is specified, matched-label and text runs use BF16
recurrence with an FP32 output head; archived native trajectories use the native
evaluator's BF16 arithmetic. Section~\ref{pd:app:extended-precision} specifies
the paired arithmetic arms, including the separate native-evaluator comparison
on $32$ questions per task and checkpoint.

\subsection{Paired uncertainty estimates}

We resample whole questions, retaining every depth, scoring variant and paired
condition of each sampled question. Unless specified otherwise, intervals are
marginal $95\%$ percentile bootstrap intervals from $5{,}000$ draws. Archive and
format comparisons describe previously inspected data; the fixed new-question
depth comparison uses the separate $128$-question collection. Precision and
auxiliary prediction comparisons use $2{,}000$ paired draws, with fitted quantities
held fixed for the latter.

\FloatBarrier

\clearpage
\begin{thebibliography}{37}
\providecommand{\natexlab}[1]{#1}
\providecommand{\url}[1]{\texttt{#1}}
\expandafter\ifx\csname urlstyle\endcsname\relax
  \providecommand{\doi}[1]{doi: #1}\else
  \providecommand{\doi}{doi: \begingroup \urlstyle{rm}\Url}\fi

\bibitem[Bae et~al.(2025{\natexlab{a}})Bae, Fisch, Harutyunyan, Ji, Kim, and
  Schuster]{bae2025relaxed}
Sangmin Bae, Adam Fisch, Hrayr Harutyunyan, Ziwei Ji, Seungyeon Kim, and Tal
  Schuster.
\newblock Relaxed recursive transformers: Effective parameter sharing with
  layer-wise {LoRA}.
\newblock In \emph{International Conference on Learning Representations},
  2025{\natexlab{a}}.
\newblock URL
  \url{https://proceedings.iclr.cc/paper_files/paper/2025/hash/54d6a55225cebbdc16fbb0e45c5bdf2b-Abstract-Conference.html}.

\bibitem[Bae et~al.(2025{\natexlab{b}})Bae, Kim, Bayat, Kim, Ha, Schuster,
  Fisch, Harutyunyan, Ji, Courville, and Yun]{bae2025mor}
Sangmin Bae, Yujin Kim, Reza Bayat, Sungnyun Kim, Jiyoun Ha, Tal Schuster, Adam
  Fisch, Hrayr Harutyunyan, Ziwei Ji, Aaron~C. Courville, and Se-Young Yun.
\newblock Mixture-of-recursions: Learning dynamic recursive depths for adaptive
  token-level computation.
\newblock In \emph{Advances in Neural Information Processing Systems},
  volume~38, 2025{\natexlab{b}}.
\newblock \doi{10.52202/085713-3229}.
\newblock URL
  \url{https://papers.nips.cc/paper_files/paper/2025/hash/8b08bbf8b420faa6eeb4020720582ec7-Abstract-Conference.html}.

\bibitem[Bai et~al.(2019)Bai, Kolter, and Koltun]{bai2019deq}
Shaojie Bai, J.~Zico Kolter, and Vladlen Koltun.
\newblock Deep equilibrium models.
\newblock In \emph{Advances in Neural Information Processing Systems},
  volume~32, 2019.
\newblock URL
  \url{https://papers.nips.cc/paper/2019/hash/01386bd6d8e091c2ab4c7c7de644d37b-Abstract.html}.

\bibitem[Belrose et~al.(2023)Belrose, Ostrovsky, McKinney, Furman, Smith,
  Halawi, Biderman, and Steinhardt]{belrose2023tunedlens}
Nora Belrose, Igor Ostrovsky, Lev McKinney, Zach Furman, Logan Smith, Danny
  Halawi, Stella Biderman, and Jacob Steinhardt.
\newblock Eliciting latent predictions from transformers with the tuned lens.
\newblock \emph{arXiv preprint arXiv:2303.08112}, 2023.
\newblock URL \url{https://arxiv.org/abs/2303.08112}.

\bibitem[Biderman et~al.(2024)Biderman, Schoelkopf, Sutawika, Gao, Tow, Abbasi,
  Aji, Ammanamanchi, Black, Clive, DiPofi, Etxaniz, Fattori, Forde, Foster,
  Hsu, Jaiswal, Lee, Li, Lovering, Muennighoff, Pavlick, Phang, Skowron, Tan,
  Tang, Wang, Winata, Yvon, and Zou]{biderman2024evaluation}
Stella Biderman, Hailey Schoelkopf, Lintang Sutawika, Leo Gao, Jonathan Tow,
  Baber Abbasi, Alham~Fikri Aji, Pawan~Sasanka Ammanamanchi, Sidney Black,
  Jordan Clive, Anthony DiPofi, Julen Etxaniz, Benjamin Fattori, Jessica~Zosa
  Forde, Charles Foster, Jeffrey Hsu, Mimansa Jaiswal, Wilson~Y. Lee, Haonan
  Li, Charles Lovering, Niklas Muennighoff, Ellie Pavlick, Jason Phang, Aviya
  Skowron, Samson Tan, Xiangru Tang, Kevin~A. Wang, Genta~Indra Winata,
  Fran{\c{c}}ois Yvon, and Andy Zou.
\newblock Lessons from the trenches on reproducible evaluation of language
  models.
\newblock \emph{arXiv preprint arXiv:2405.14782}, 2024.
\newblock URL \url{https://arxiv.org/abs/2405.14782}.

\bibitem[Blanchard et~al.(2021)Blanchard, Higham, and
  Higham]{blanchard2021logsumexp}
Pierre Blanchard, Desmond~J. Higham, and Nicholas~J. Higham.
\newblock Accurately computing the log-sum-exp and softmax functions.
\newblock \emph{IMA Journal of Numerical Analysis}, 41\penalty0 (4):\penalty0
  2311--2330, 2021.
\newblock \doi{10.1093/imanum/draa038}.
\newblock URL \url{https://doi.org/10.1093/imanum/draa038}.

\bibitem[Chuang et~al.(2024)Chuang, Xie, Luo, Kim, Glass, and
  He]{chuang2024dola}
Yung-Sung Chuang, Yujia Xie, Hongyin Luo, Yoon Kim, James Glass, and Pengcheng
  He.
\newblock {DoLa}: Decoding by contrasting layers improves factuality in large
  language models.
\newblock In \emph{International Conference on Learning Representations}, 2024.
\newblock URL
  \url{https://proceedings.iclr.cc/paper_files/paper/2024/hash/edc36117f795ca52a0cbf6a7b3882859-Abstract-Conference.html}.

\bibitem[Clark et~al.(2018)Clark, Cowhey, Etzioni, Khot, Sabharwal, Schoenick,
  and Tafjord]{clark2018arc}
Peter Clark, Isaac Cowhey, Oren Etzioni, Tushar Khot, Ashish Sabharwal, Carissa
  Schoenick, and Oyvind Tafjord.
\newblock Think you have solved question answering? try {ARC}, the {AI2}
  reasoning challenge.
\newblock \emph{arXiv preprint arXiv:1803.05457}, 2018.
\newblock URL \url{https://arxiv.org/abs/1803.05457}.

\bibitem[Dehghani et~al.(2019)Dehghani, Gouws, Vinyals, Uszkoreit, and
  Kaiser]{dehghani2019universal}
Mostafa Dehghani, Stephan Gouws, Oriol Vinyals, Jakob Uszkoreit, and Lukasz
  Kaiser.
\newblock Universal transformers.
\newblock In \emph{International Conference on Learning Representations}, 2019.
\newblock URL \url{https://openreview.net/forum?id=HyzdRiR9Y7}.

\bibitem[Egozcue et~al.(2003)Egozcue, Pawlowsky-Glahn, Mateu-Figueras, and
  Barcel{\'o}-Vidal]{egozcue2003logratio}
J.~J. Egozcue, V.~Pawlowsky-Glahn, G.~Mateu-Figueras, and C.~Barcel{\'o}-Vidal.
\newblock Isometric logratio transformations for compositional data analysis.
\newblock \emph{Mathematical Geology}, 35\penalty0 (3):\penalty0 279--300,
  2003.
\newblock \doi{10.1023/A:1023818214614}.
\newblock URL \url{https://doi.org/10.1023/A:1023818214614}.

\bibitem[Elhoushi et~al.(2024)Elhoushi, Shrivastava, Liskovich, Hosmer, Wasti,
  Lai, Mahmoud, Acun, Agarwal, Roman, Aly, Chen, and Wu]{elhoushi2024layerskip}
Mostafa Elhoushi, Akshat Shrivastava, Diana Liskovich, Basil Hosmer, Bram
  Wasti, Liangzhen Lai, Anas Mahmoud, Bilge Acun, Saurabh Agarwal, Ahmed Roman,
  Ahmed Aly, Beidi Chen, and Carole-Jean Wu.
\newblock {LayerSkip}: Enabling early exit inference and self-speculative
  decoding.
\newblock In \emph{Proceedings of the 62nd Annual Meeting of the Association
  for Computational Linguistics (Volume 1: Long Papers)}, pages 12622--12642,
  2024.
\newblock \doi{10.18653/v1/2024.acl-long.681}.
\newblock URL \url{https://aclanthology.org/2024.acl-long.681/}.

\bibitem[Gao and Pavel(2017)]{gao2018softmax}
Bolin Gao and Lacra Pavel.
\newblock On the properties of the softmax function with application in game
  theory and reinforcement learning.
\newblock \emph{arXiv preprint arXiv:1704.00805}, 2017.
\newblock URL \url{https://arxiv.org/abs/1704.00805}.

\bibitem[Gaubert and Qu(2015)]{gaubert2015dobrushin}
St{\'e}phane Gaubert and Zheng Qu.
\newblock Dobrushin's ergodicity coefficient for markov operators on cones.
\newblock \emph{Integral Equations and Operator Theory}, 81\penalty0
  (1):\penalty0 127--150, 2015.
\newblock \doi{10.1007/s00020-014-2193-2}.
\newblock URL \url{https://arxiv.org/abs/1307.4649}.

\bibitem[Geiping et~al.(2025)Geiping, McLeish, Jain, Kirchenbauer, Singh,
  Bartoldson, Kailkhura, Bhatele, and Goldstein]{geiping2025latent}
Jonas Geiping, Sean McLeish, Neel Jain, John Kirchenbauer, Siddharth Singh,
  Brian~R. Bartoldson, Bhavya Kailkhura, Abhinav Bhatele, and Tom Goldstein.
\newblock Scaling up test-time compute with latent reasoning: A recurrent depth
  approach.
\newblock In \emph{Advances in Neural Information Processing Systems},
  volume~38, 2025.
\newblock \doi{10.52202/085713-1380}.
\newblock URL
  \url{https://proceedings.neurips.cc/paper_files/paper/2025/hash/3b01972cf31e6fa0fe29e4b8b5c2a0a1-Abstract-Conference.html}.

\bibitem[Geva et~al.(2022)Geva, Caciularu, Wang, and
  Goldberg]{geva2022vocabulary}
Mor Geva, Avi Caciularu, Kevin Wang, and Yoav Goldberg.
\newblock Transformer feed-forward layers build predictions by promoting
  concepts in the vocabulary space.
\newblock In \emph{Proceedings of the 2022 Conference on Empirical Methods in
  Natural Language Processing}, pages 30--45. Association for Computational
  Linguistics, 2022.
\newblock \doi{10.18653/v1/2022.emnlp-main.3}.
\newblock URL \url{https://aclanthology.org/2022.emnlp-main.3/}.

\bibitem[Giannou et~al.(2023)Giannou, Rajput, Sohn, Lee, Lee, and
  Papailiopoulos]{giannou2023programmable}
Angeliki Giannou, Shashank Rajput, Jy-Yong Sohn, Kangwook Lee, Jason~D. Lee,
  and Dimitris Papailiopoulos.
\newblock Looped transformers as programmable computers.
\newblock In \emph{Proceedings of the 40th International Conference on Machine
  Learning}, volume 202 of \emph{Proceedings of Machine Learning Research},
  pages 11398--11442. PMLR, 2023.
\newblock URL \url{https://proceedings.mlr.press/v202/giannou23a.html}.

\bibitem[Goyal et~al.(2024)Goyal, Ji, Rawat, Menon, Kumar, and
  Nagarajan]{goyal2024pause}
Sachin Goyal, Ziwei Ji, Ankit~Singh Rawat, Aditya~Krishna Menon, Sanjiv Kumar,
  and Vaishnavh Nagarajan.
\newblock Think before you speak: Training language models with pause tokens.
\newblock In \emph{International Conference on Learning Representations}, 2024.
\newblock URL
  \url{https://proceedings.iclr.cc/paper_files/paper/2024/hash/76917808731dae9e6d62c2a7a6afb542-Abstract-Conference.html}.

\bibitem[Graves(2016)]{graves2016act}
Alex Graves.
\newblock Adaptive computation time for recurrent neural networks.
\newblock \emph{arXiv preprint arXiv:1603.08983}, 2016.
\newblock URL \url{https://arxiv.org/abs/1603.08983}.

\bibitem[Hao et~al.(2025)Hao, Sukhbaatar, Su, Li, Hu, Weston, and
  Tian]{hao2025coconut}
Shibo Hao, Sainbayar Sukhbaatar, DiJia Su, Xian Li, Zhiting Hu, Jason Weston,
  and Yuandong Tian.
\newblock Training large language models to reason in a continuous latent
  space.
\newblock In \emph{Conference on Language Modeling}, 2025.
\newblock URL \url{https://arxiv.org/abs/2412.06769}.

\bibitem[Hein and Andriushchenko(2017)]{hein2017formal}
Matthias Hein and Maksym Andriushchenko.
\newblock Formal guarantees on the robustness of a classifier against
  adversarial manipulation.
\newblock In \emph{Advances in Neural Information Processing Systems},
  volume~30, 2017.
\newblock URL
  \url{https://proceedings.neurips.cc/paper_files/paper/2017/hash/e077e1a544eec4f0307cf5c3c721d944-Abstract.html}.

\bibitem[Hendrycks et~al.(2021)Hendrycks, Burns, Basart, Zou, Mazeika, Song,
  and Steinhardt]{hendrycks2021mmlu}
Dan Hendrycks, Collin Burns, Steven Basart, Andy Zou, Mantas Mazeika, Dawn
  Song, and Jacob Steinhardt.
\newblock Measuring massive multitask language understanding.
\newblock In \emph{International Conference on Learning Representations}, 2021.
\newblock URL \url{https://openreview.net/forum?id=d7KBjmI3GmQ}.

\bibitem[Holtzman et~al.(2021)Holtzman, West, Shwartz, Choi, and
  Zettlemoyer]{holtzman2021surface}
Ari Holtzman, Peter West, Vered Shwartz, Yejin Choi, and Luke Zettlemoyer.
\newblock Surface form competition: Why the highest probability answer isn't
  always right.
\newblock In \emph{Proceedings of the 2021 Conference on Empirical Methods in
  Natural Language Processing}, pages 7038--7051. Association for Computational
  Linguistics, 2021.
\newblock \doi{10.18653/v1/2021.emnlp-main.564}.
\newblock URL \url{https://aclanthology.org/2021.emnlp-main.564/}.

\bibitem[Jazbec et~al.(2024)Jazbec, Timans, Had{\v z}i~Veljkovi{\'c}, Sakmann,
  Zhang, Naesseth, and Nalisnick]{jazbec2024fast}
Metod Jazbec, Alexander Timans, Tin Had{\v z}i~Veljkovi{\'c}, Kaspar Sakmann,
  Dan Zhang, Christian~A. Naesseth, and Eric Nalisnick.
\newblock Fast yet safe: Early-exiting with risk control.
\newblock In \emph{Advances in Neural Information Processing Systems},
  volume~37, 2024.
\newblock \doi{10.52202/079017-4124}.
\newblock URL
  \url{https://proceedings.neurips.cc/paper_files/paper/2024/hash/ea5a63f7ddb82e58623693fd1f4933f7-Abstract-Conference.html}.

\bibitem[Jeddi et~al.(2026)Jeddi, Ciccone, and Taati]{jeddi2026loopformer}
Ahmadreza Jeddi, Marco Ciccone, and Babak Taati.
\newblock {LoopFormer}: Elastic-depth looped transformers for latent reasoning
  via shortcut modulation.
\newblock In \emph{International Conference on Learning Representations}, 2026.
\newblock URL \url{https://iclr.cc/virtual/2026/poster/10009450}.

\bibitem[Liu et~al.(2024)Liu, Tang, Liu, Ni, Tang, Han, and
  Wang]{liu2024kangaroo}
Fangcheng Liu, Yehui Tang, Zhenhua Liu, Yunsheng Ni, Duyu Tang, Kai Han, and
  Yunhe Wang.
\newblock Kangaroo: Lossless self-speculative decoding for accelerating {LLM}s
  via double early exiting.
\newblock In \emph{Advances in Neural Information Processing Systems},
  volume~37, 2024.
\newblock \doi{10.52202/079017-0381}.
\newblock URL
  \url{https://proceedings.neurips.cc/paper_files/paper/2024/hash/16336d94a5ffca8de019087ab7fe403f-Abstract-Conference.html}.

\bibitem[Liu and Wang(2025)]{liu2025answer}
Xin Liu and Lu~Wang.
\newblock Answer convergence as a signal for early stopping in reasoning.
\newblock In \emph{Proceedings of the 2025 Conference on Empirical Methods in
  Natural Language Processing}, pages 17896--17907. Association for
  Computational Linguistics, 2025.
\newblock \doi{10.18653/v1/2025.emnlp-main.904}.
\newblock URL \url{https://aclanthology.org/2025.emnlp-main.904/}.

\bibitem[McLeish et~al.(2025)McLeish, Li, Kirchenbauer, Kalra, Bartoldson,
  Kailkhura, Schwarzschild, Geiping, Goldstein, and
  Goldblum]{mcleish2025recurrent}
Sean McLeish, Ang Li, John Kirchenbauer, Dayal~Singh Kalra, Brian~R.
  Bartoldson, Bhavya Kailkhura, Avi Schwarzschild, Jonas Geiping, Tom
  Goldstein, and Micah Goldblum.
\newblock Teaching pretrained language models to think deeper with retrofitted
  recurrence.
\newblock \emph{arXiv preprint arXiv:2511.07384}, 2025.
\newblock URL \url{https://arxiv.org/abs/2511.07384}.

\bibitem[Robinson et~al.(2023)Robinson, Rytting, and
  Wingate]{robinson2023multiplechoice}
Joshua Robinson, Christopher~Michael Rytting, and David Wingate.
\newblock Leveraging large language models for multiple choice question
  answering.
\newblock In \emph{International Conference on Learning Representations}, 2023.
\newblock URL \url{https://arxiv.org/abs/2210.12353}.

\bibitem[Saunshi et~al.(2025)Saunshi, Dikkala, Li, Kumar, and
  Reddi]{saunshi2025latent}
Nikunj Saunshi, Nishanth Dikkala, Zhiyuan Li, Sanjiv Kumar, and Sashank~J.
  Reddi.
\newblock Reasoning with latent thoughts: On the power of looped transformers.
\newblock In \emph{International Conference on Learning Representations}, 2025.
\newblock URL
  \url{https://proceedings.iclr.cc/paper_files/paper/2025/hash/2676109d49d1eb26d6bc584a8f556305-Abstract-Conference.html}.

\bibitem[Schuster et~al.(2022)Schuster, Fisch, Gupta, Dehghani, Bahri, Tran,
  Tay, and Metzler]{schuster2022calm}
Tal Schuster, Adam Fisch, Jai Gupta, Mostafa Dehghani, Dara Bahri, Vinh~Q.
  Tran, Yi~Tay, and Donald Metzler.
\newblock Confident adaptive language modeling.
\newblock In \emph{Advances in Neural Information Processing Systems},
  volume~35, 2022.
\newblock URL
  \url{https://papers.nips.cc/paper_files/paper/2022/hash/6fac9e316a4ae75ea244ddcef1982c71-Abstract-Conference.html}.

\bibitem[Snell et~al.(2025)Snell, Lee, Xu, and Kumar]{snell2025testtime}
Charlie Snell, Jaehoon Lee, Kelvin Xu, and Aviral Kumar.
\newblock Scaling {LLM} test-time compute optimally can be more effective than
  scaling parameters for reasoning.
\newblock In \emph{International Conference on Learning Representations}, 2025.
\newblock URL
  \url{https://proceedings.iclr.cc/paper_files/paper/2025/hash/1b623663fd9b874366f3ce019fdfdd44-Abstract-Conference.html}.
\newblock Oral presentation.

\bibitem[Tsuzuku et~al.(2018)Tsuzuku, Sato, and Sugiyama]{tsuzuku2018lipschitz}
Yusuke Tsuzuku, Issei Sato, and Masashi Sugiyama.
\newblock Lipschitz-margin training: Scalable certification of perturbation
  invariance for deep neural networks.
\newblock In \emph{Advances in Neural Information Processing Systems},
  volume~31, 2018.
\newblock URL
  \url{https://proceedings.neurips.cc/paper_files/paper/2018/hash/485843481a7edacbfce101ecb1e4d2a8-Abstract.html}.

\bibitem[Xin et~al.(2020)Xin, Tang, Lee, Yu, and Lin]{xin2020deebert}
Ji~Xin, Raphael Tang, Jaejun Lee, Yaoliang Yu, and Jimmy Lin.
\newblock {DeeBERT}: Dynamic early exiting for accelerating {BERT} inference.
\newblock In \emph{Proceedings of the 58th Annual Meeting of the Association
  for Computational Linguistics}, pages 2246--2251, 2020.
\newblock \doi{10.18653/v1/2020.acl-main.204}.
\newblock URL \url{https://aclanthology.org/2020.acl-main.204/}.

\bibitem[Yang et~al.(2026)Yang, Han, Zhang, Wei, Shao, Guo, and
  Li]{yang2026stars}
Xiao-Wen Yang, Ziyu Han, Xi-Hua Zhang, Wen-Da Wei, Jie-Jing Shao, Lan-Zhe Guo,
  and Yu-Feng Li.
\newblock Stabilizing recurrent dynamics for test-time scalable latent
  reasoning in looped language models.
\newblock In \emph{International Conference on Machine Learning}, 2026.
\newblock URL \url{https://arxiv.org/abs/2605.26733}.

\bibitem[Yom~Din et~al.(2024)Yom~Din, Karidi, Choshen, and
  Geva]{yomdin2024jump}
Alexander Yom~Din, Taelin Karidi, Leshem Choshen, and Mor Geva.
\newblock Jump to conclusions: Short-cutting transformers with linear
  transformations.
\newblock In \emph{Proceedings of the 2024 Joint International Conference on
  Computational Linguistics, Language Resources and Evaluation}, pages
  9615--9625. ELRA and ICCL, 2024.
\newblock URL \url{https://aclanthology.org/2024.lrec-main.840/}.

\bibitem[Zheng et~al.(2024)Zheng, Zhou, Meng, Zhou, and
  Huang]{zheng2024selectors}
Chujie Zheng, Hao Zhou, Fandong Meng, Jie Zhou, and Minlie Huang.
\newblock Large language models are not robust multiple choice selectors.
\newblock In \emph{International Conference on Learning Representations}, 2024.
\newblock URL
  \url{https://proceedings.iclr.cc/paper_files/paper/2024/hash/54dd9e0cff6d9214e20d97eb2a3bae49-Abstract-Conference.html}.

\bibitem[Zhu et~al.(2025)Zhu, Wang, Hua, et~al.]{zhu2025ouro}
Rui-Jie Zhu, Zixuan Wang, Kai Hua, et~al.
\newblock Scaling latent reasoning via looped language models.
\newblock \emph{arXiv preprint arXiv:2510.25741}, 2025.
\newblock URL \url{https://arxiv.org/abs/2510.25741}.

\end{thebibliography}
\end{document}